\documentclass{article}

\usepackage{pgfplots}
\pgfplotsset{compat=1.18}

\PassOptionsToPackage{numbers}{natbib}
\usepackage[preprint]{neurips_2025}

\usepackage[utf8]{inputenc} 
\usepackage[T1]{fontenc}    
\usepackage{hyperref}       
\usepackage{url}            
\usepackage{booktabs}       
\usepackage{amsfonts}       
\usepackage{nicefrac}       
\usepackage{microtype}      
\usepackage{xcolor}         
\usepackage{enumitem} 
\usepackage{graphicx}
\usepackage{amsmath}
\usepackage[capitalize,noabbrev]{cleveref}
\usepackage{amsthm}  

\usepackage{multirow}

\newcommand{\bfD}{{\bf D}}

\newcommand{\bfH}{{\bf H}}

\newcommand{\bfM}{{\bf M}}
\newcommand{\bfN}{{\bf N}}

\newcommand{\bfh}{{\bf h}}

\usepackage{xcolor,colortbl}
\usepackage{xcolor}

\definecolor{ForestGreen}{HTML}{3F8A2D}    
\definecolor{RoyalBlue}{HTML}{000AFF}     
\definecolor{CustomOrange}{HTML}{F2AD4E}   

\usepackage{subcaption}
\usepackage{wrapfig}

\newcommand{\ie}{i.e., }

\definecolor{rebuttalcolor}{HTML}{FF6E00} 

\definecolor{goodcell}{HTML}{bfd9c1}
\definecolor{badcell}{HTML}{e3a6a6}
\definecolor{first}{HTML}{00A64F}  
\definecolor{second}{HTML}{006EB8} 
\newcommand{\one}[1]{\textcolor{first}{\bf#1}}
\newcommand{\two}[1]{\textcolor{second}{\bf#1}}
\newcommand{\three}[1]{{\bf#1}}

\usepackage{multirow}
\usepackage{url}
\usepackage{soul}
\usepackage{tablefootnote}
\usepackage{enumitem}

\newtheorem{theorem}{Theorem}
\newtheorem{definition}{Definition}
\newtheorem{proposition}{Proposition}

\newcommand{\ourmethod}{\textsc{Tango}}

\title{\ourmethod{}: Graph Neural Dynamics
via Learned Energy and Tangential Flows}

\author{%
  Moshe Eliasof \\
  University of Cambridge\\
  United Kingdom \\
  \texttt{me532@cam.ac.uk} \\
\And 
  Eldad Haber \\
  University of British Columbia\\
  Canada \\
  \texttt{ehaber@eoas.ubc.ca} \\ 
  \And 
  Carola-Bibiane Sch\"onlieb \\
  University of Cambridge\\
  United Kingdom \\
  \texttt{cbs31@cam.ac.uk} \\ 
}

\begin{document}

\maketitle

\begin{abstract}
We introduce \ourmethod{} --  a dynamical systems inspired framework for graph representation learning that governs node feature evolution through a learned energy landscape and its associated descent dynamics. At the core of our approach is a learnable Lyapunov function over node embeddings, whose gradient defines an energy-reducing direction that guarantees convergence and stability. To enhance flexibility while preserving the benefits of energy-based dynamics, we incorporate a novel tangential component, learned via message passing, that evolves features while maintaining the energy value.
This decomposition into orthogonal flows of energy gradient descent and tangential evolution yields a flexible form of graph dynamics, and enables effective signal propagation even in flat or ill-conditioned energy regions, that often appear in graph learning. Our method mitigates oversquashing and is compatible with different graph neural network backbones. Empirically, \ourmethod{} achieves strong performance across a diverse set of node and graph classification and regression benchmarks, demonstrating the effectiveness of jointly learned energy functions and tangential flows for graph neural networks.
\end{abstract}

\section{Introduction}
\label{sec:intro}
Graph Neural Networks (GNNs) have achieved remarkable success in learning representations for graph-structured data \citep{bronstein2021geometric}, but they face fundamental challenges when scaling depth or modeling long-range interactions, such as vanishing gradients \citep{arroyo2025vanishing}, over-smoothing \citep{nt2019revisiting, cai2020note, rusch2023survey}, and over-squashing \cite{alon2021on,topping2022understanding,diGiovanniOversquashing,gravina_adgn,gravina_swan}. To address these issues, recent works have drawn connections between GNNs and dynamical systems or control theory to understand and mitigate these issues \cite{poli2019gnnode,chamberlain2021grand, eliasof2021pde,gravina_adgn,arroyo2025vanishing}. For example, treating a GNN as a continuous dynamical system (or \emph{neural ODE}) opens the door to analyzing stability through the lens of diffusion \citep{chamberlain2021grand}, energy conservation \citep{rusch2022graph},  antisymmetric dynamics \cite{gravina_adgn}, and Hamiltonian flows \citep{gravina_phdgn}. In parallel, physics-informed neural architectures have shown that embedding physical priors such as energy conservation or dissipation into neural models can dramatically improve stability and interpretability  \cite{bhattoo2022lgnn, gao2022physics, brandstetter2022message}. The common theme in the aforementioned works is the reliance on the existence of \emph{some} energy functional that is minimized or preserved by the GNN parameterization, which is often relatively simple, such as the Dirichlet energy \citep{rusch2023survey}. 

At the same time, it is well-established in bioinformatics and computational chemistry that different, and more complex, energy functions are necessary to accurately model various natural processes. For instance, in protein folding, the energy landscape is often rugged and multi-funnel-shaped, reflecting the presence of multiple stable conformations and transition pathways \citep{wolynes2005recent}. Similarly, in computational chemistry, modeling complex chemical reactions and molecular interactions requires sophisticated potential energy surfaces \citep{senn2009qm}.

In recent years, there has been a growing body of work on \emph{energy-based models} (EBMs) in deep learning. These models aim to learn an energy function that captures the probability distribution of data, such as images or molecules, primarily for generative modeling purposes \citep{lecun2006tutorial, xie2016theory, du2019implicit, guo2023egc}. In contrast, our work focuses on learning a \emph{downstream task-driven energy} that models the energy landscape whose minimization corresponds to the solution of a downstream task, such as graph or node classification, rather than focusing on generative modeling.

These insights motivate a fundamental question: \emph{How can we learn a task-driven energy function, and how can it be effectively leveraged within a GNN architecture to guide representation dynamics?} Unlike energy-based generative models, where the energy function encodes data likelihood, our focus is on learning an energy landscape whose minimization corresponds to solving a downstream task, such as node or graph classification. 

 To address these questions, we propose to decompose feature evolution into two orthogonal components: (i) a \emph{gradient descent} direction that minimizes the learned energy, and (ii) a \emph{tangential direction} that evolves along its level sets, preserving energy. This structured decomposition yields a principled framework that promotes stability, enhances interpretability, and mitigates issues such as oversquashing.

\textbf{Our Approach.} We introduce \ourmethod{}, a framework for constrained graph dynamics that incorporates a learnable Lyapunov energy function into the message-passing process. The learned energy governs representation updates through the two complementary flows described above: (1) an \textit{energy descent component}, which ensures convergence toward task-relevant solutions, and (2) a \textit{tangential, conservative component}, which retains flexibility while preserving energy. As illustrated in Figure~\ref{fig:lyapunov-dynamics}, the descent direction (green) lowers the energy, while the tangential direction (blue) moves along energy level sets. Their combination (orange) defines the full update step. This structured decomposition allows \ourmethod{} to propagate information effectively while maintaining controlled and stable feature dynamics.

\begin{figure}[t]
  \centering
  \includegraphics[width=0.70\linewidth]{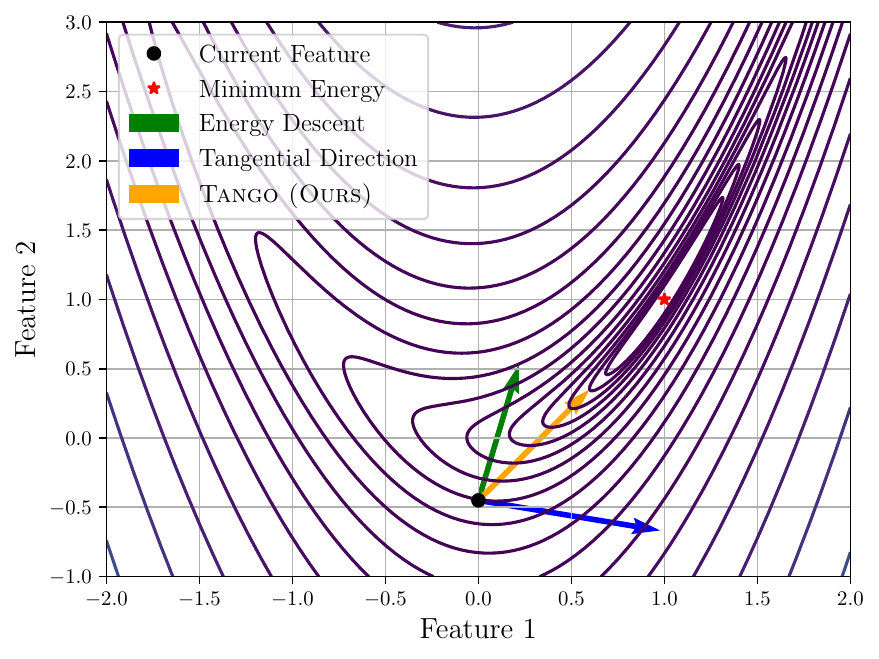}
  \caption{
Illustration of \ourmethod{} dynamics in a 2D feature space. We plot the level sets of a learned energy function and visualize the \textcolor{ForestGreen}{energy descent direction (green)}, the \textcolor{RoyalBlue}{learned tangential direction (blue)}, and their \textcolor{CustomOrange}{combined vector (orange)}. The tangential component enables movement along level sets, while the descent component reduces energy. Together, they allow for more effective navigation of the learned energy landscape.}
  \label{fig:lyapunov-dynamics}
\end{figure}

\clearpage
\textbf{Main Contributions.} Our contributions are as follows:
\begin{enumerate}[leftmargin=1.5em]
\item \textbf{Lyapunov-inspired Graph Neural Dynamics.} We introduce \ourmethod{}, a novel framework for graph representation learning that decomposes feature evolution into energy descent and tangential components, both parameterized by GNNs.
\item \textbf{Theoretical Guarantees.} We prove that, under mild assumptions, \ourmethod{} satisfies Lyapunov conditions, ensuring stable dynamics. Additionally, we show that the tangential component helps mitigate oversquashing by enabling expressive yet controlled propagation.
\item \textbf{Strong Empirical Performance.} We evaluate \ourmethod{} on a range of graph learning benchmarks, demonstrating performance competitive with or surpassing strong and widely-used baselines.
\end{enumerate}

\section{Mathematical Background}
\label{sec:background}

In this section, we provide a brief overview of Lyapunov stability theory, based on the classical treatment in \citet{khalil2002nonlinear}, which underpins the design of our \ourmethod. This theory originates from control systems and differential equations, offering a principled way to assess whether trajectories of a dynamical system remain bounded and converge over time.  

\textbf{Continuous Dynamical Systems.} 
Let $\bfh(t) \in \mathbb{R}^{d}$ denote the state of a dynamical system at time $t \geq 0$, and consider a  first-order ODE:
\begin{equation}
    \frac{d\bfh(t)}{dt} = F(\bfh(t)), \label{eq:dyn_sys}
\end{equation}
where $F: \mathbb{R}^{d} \to \mathbb{R}^{d}$ is a continuous vector field. A point $\bfh^*$ is called an \emph{equilibrium} if $F(\bfh^*) = 0$.

\begin{definition}[Lyapunov Function]
Let $\bfh^* \in \mathbb{R}^{d}$ be an equilibrium of the system in \Cref{eq:dyn_sys}. A continuously differentiable function $V: \mathbb{R}^{d} \to \mathbb{R}$ is called a \emph{Lyapunov function} around $\bfh^*$ if:
\begin{enumerate}
    \item $V(\bfh) \ge 0$ for all $\bfh$ in a neighborhood of $\bfh^*$, and $V(\bfh^*) = 0$;
    \item $\frac{d}{dt}V(\bfh(t)) = \nabla_\bfh V(\bfh(t))^\top F(\bfh(t)) \le 0$ in that neighborhood.
\end{enumerate}
\end{definition}

The first condition ensures that $V$ is lower-bounded by 0, i.e., that value of the Lyapunov function, sometimes also referred to as \emph{energy} is non-negative,  and the second that $V$ does not increase along trajectories of the system.

We now recall a classical \citep{khalil2002nonlinear} stability criterion for the dynamical system in \Cref{eq:dyn_sys}, based on the definition of a Lyapunov function, which we will later use to characterize the stability of our approach in \Cref{sec:theoretical-analysis}.
\begin{theorem}[Lyapunov Stability]
Let $\bfh^*$ be an equilibrium of \Cref{eq:dyn_sys} and let $V$ be a Lyapunov function in a neighborhood $\mathcal{N}$ of $\bfh^*$. If $\frac{d}{dt}V(\bfh(t)) \le 0$ in $\mathcal{N}$, then $\bfh^*$ is \emph{Lyapunov stable}.
\label{thm:lyapunov}
\end{theorem}

\section{Method}
\label{sec:method}

As discussed in \Cref{sec:intro}, our goal is to learn a task-driven energy function, and to devise a principled way to utilize it towards improved downstream performance for graph learning tasks, based on a combination of \underline{\textsc{Tan}}gential- and \underline{\textsc{G}}radient-steps \underline{\textsc{O}}ptimization of node features. We therefore call our method \ourmethod{}. In \Cref{sec:overview_method}, we outline the blueprint of  \ourmethod{}. In \Cref{sec:lyapunov_param}, we discuss implementation details. Later, in \Cref{sec:theoretical-analysis}, we discuss the properties of our \ourmethod{}, and in \Cref{app:runtimes_complexity} we discuss its complexity.

\textbf{Notations.} We consider a graph $\mathcal{G} = (\mathcal{V}, \mathcal{E})$ with $n = |\mathcal{V}|$ nodes and $m = |\mathcal{E}|$ edges. Let $\bfH(t) = [\bfh_1(t), \bfh_2(t), \dots, \bfh_n(t)]^\top \in \mathbb{R}^{n \times d}$ denote the matrix of node features at continuous time $t$, where $\bfh_v(t) \in \mathbb{R}^d$ is the state of node $v$ at time $t$.  Following the literature of GNNs based on dynamical systems \citep{eliasof2021pde, gravina_adgn, arroyo2025vanishing}, when considering a discrete architecture with a finite number of layers, we draw an analogy between time $t$ and network depth $\ell$. Henceforth, we will interchangeably use the terms $\bfH(t)$ and $\bfH^{(\ell)}$ to denote node features at a certain time or layer of the network, depending on the context.

\subsection{Optimizing Features with Energy Tangential and Gradient Steps}
\label{sec:overview_method}
Our \ourmethod{} concept is based on a dynamical system that, given a graph energy function $V_{\cal{G}}$, considers two steps: (i) \emph{energy gradient descent} and (ii) \emph{tangential direction} flows, that evolve the node features:
\begin{equation}
\frac{d\bfH(t)}{dt} = \underbrace{-\alpha_{\mathcal{G}}(\bfH(t)) \nabla_{\bfH} V_{\mathcal{G}}(\bfH(t))}_{\text{Energy Gradient Descent}} + \underbrace{\beta_{\mathcal{G}}(\bfH(t))\, T_{V_{\mathcal{G}}}(\bf   H(t))}_{\text{Tangential Direction}},
\label{eq:lyapunov_gnn_ode}
\end{equation}

where $\alpha_{\mathcal{G}}, \beta_{\mathcal{G}}$ are non-negative scalars that balance the two steps, $\nabla_{\bfH}V_{\mathcal{G}}(\bfH(t))$ is the energy gradient with respect to node features $\bfH(t)$, and $T_{V_{\mathcal{G}}}(\bfH(t))$ is an update direction that is orthogonal, i.e, tangential to the energy gradient. We note that, while in general, there are many possible directions that are orthogonal to the gradient, in \Cref{sec:lyapunov_param} we specify a procedure for learning this direction. In particular, we note that, by design, the first step decreases the energy, while the second is a tangential flow that preserves energy. Below, we formalize the tangential component and provide implementation details in Sections~\ref{sec:lyapunov_param}.

\textbf{Tangential Flow.} Setting $\beta_{\mathcal{G}}=0$ in \Cref{eq:lyapunov_gnn_ode} yields a standard energy gradient flow applied to the features.  While it guarantees energy dissipation, it may suffer from slow convergence \citep{boyd2004convex, nw} and restricted dynamics during training.  As discussed in \Cref{sec:intro}, while a gradient flow is commonly used in generative applications, accompanied by hundreds or thousands of steps are, this approach is not suitable for downstream learning, as it renders a neural network with equivalently many effective layers, that is hard to train \citep{peng2024beyond} and high computational costs. To address this, and to accelerate the minimization of the energy function, we introduce a \emph{tangential} flow that evolves tangentially to the gradient of $V_{\mathcal{G}}$, preserving energy. As we illustrate in \Cref{fig:lyapunov-dynamics}, and later theoretically discuss in \Cref{sec:theoretical-analysis}, while the tangential flow itself maintains the same energy level, its combination with the energy gradient descent step, as shown in \Cref{eq:lyapunov_gnn_ode}, can offer a better overall descent direction, thereby accelerating energy convergence. 

In order to obtain a direction that is orthogonal to $\nabla_{\bfH}V_{\mathcal{G}}(\bfH(t))$,  let $\bfM(\bfH(t))$ be a predicted update direction of the node features. We then define the \emph{tangential} node feature update direction as:
\begin{equation}
T_{V_{\cal{G}}}(\bfH(t)) = \bfM(\bfH(t)) - \left\langle {\bfM}(\bfH(t)), \widehat{\nabla}_{\bfH} V_{\mathcal{G}}(\bfH(t)) \right\rangle \cdot \nabla_{\bfH} V_{\mathcal{G}}(\bfH(t)),
\label{eq:tangential_flow}
\end{equation}
where $\widehat{\nabla}_{\bfH} V_{\mathcal{G}}(\bfH(t))$ is the normalized energy gradient. Unless ${\nabla}_{\bfH} V_{\mathcal{G}}(\bfH(t))=0$, where then we define $T_{V_{\cal{G}}}(\bfH(t))=\bfM(\bfH(t))$, the projection in \Cref{eq:tangential_flow}  removes shared the component of $\bfM(\bfH(t))$ with the energy descent direction, ensuring $T_{V_{\cal{G}}}$ is orthogonal to the gradient of the energy function $V_{\mathcal{G}}(\bfH(t))$. 

\subsection{\ourmethod{}  Graph Neural Networks}
\label{sec:lyapunov_param}
In \Cref{sec:overview_method}, we described the concept of \ourmethod{} and its underlying continuous dynamical system. To materialize this concept and obtain a GNN, we discretize \Cref{eq:lyapunov_gnn_ode} using the commonly used in GNNs \citep{gravina_adgn, eliasof2021pde, chamberlain2021grand, arroyo2025vanishing, choi2022gread}, forward Euler approach to obtain the following graph neural layer:
\begin{equation}
\bfH^{(\ell+1)} = \bfH^{(\ell)} + \epsilon  \left( -\alpha_{\mathcal{G}}(\bfH^{(\ell)})\, \nabla_{\bfH} V_{\mathcal{G}}(\bfH^{(\ell)}) + \beta_{\mathcal{G}}(\bfH^{(\ell)})\, T_{V_{\mathcal{G}}}(\bfH^{(\ell)})\right),
\label{eq:euler_update}
\end{equation}
for $\ell = 0, \dots, L-1$, where $\epsilon > 0$ is a hyperparameter step size that stems from the forward Euler discretization, $\nabla_{\bfH} V_{\mathcal{G}}(\bfH^{(\ell)})$ is the gradient of the energy function defined in \Cref{eq:energy_NN}. The coefficients $\alpha_{\mathcal{G}}\geq0, \ \beta_{\mathcal{G}} $ are scalars that balance the energy descent and tangential terms, and are also predicted by the respective GNNs, as discussed below. 

\textbf{Energy Function.} We now describe the implementation of the  function $V_{\mathcal{G}}$. Given features $\bfH^{(\ell)}$, we apply:
\begin{equation}
\label{eq:intermediate_energy_features}
\tilde{\bfH}^{(\ell)} = \sigma\left( \textsc{EnergyGNN}(\bfH^{(\ell)}; \mathcal{G}) \right) \in \mathbb{R}^{n \times d},
\end{equation}
where $\textsc{EnergyGNN}$ is a graph neural network (e.g., GatedGCN~\citep{gatedgcn}, GPS~\citep{rampasek2022graphgps}), and $\sigma$ is a pointwise nonlinearity. We then compute per-node energy scores using a multilayer perceptron (MLP):
\begin{equation}
\tilde{V}_{\mathcal{G}}(\tilde{\bfH}^{(\ell)}) = \text{MLP}_{\text{E}}(\tilde{\bfH}^{(\ell)}) \in \mathbb{R}^{n \times 1},
\end{equation}
and define the overall graph energy scalar value as:
\begin{equation}
\label{eq:energy_NN}
V_{\mathcal{G}}({\bfH}^{(\ell)}) =  \frac{1}{n} \sum_{v \in \mathcal{V}} \tilde{V}_{\mathcal{G}}(\tilde{\bfH}^{(\ell)})_v^2  \in \mathbb{R}_{\geq 0}.
\end{equation}

In addition, we employ a global sum pooling \citep{xu2019how} to $\tilde{\bfH}^{(\ell)}$, followed by an MLP and sigmoid activation, to obtain a bounded non-negative scalar $\alpha_{\cal{G}}$, as follows:
\begin{equation}
    \label{eq:alpha}
    \alpha_{\cal{G}}(\bfH^{(l)}) = \textsc{Sigmoid}\left(\text{MLP}_{{\alpha}}\left(\textsc{SumPool}(\tilde{\bfH}^{(\ell)})\right)\right) \in [0,1]
\end{equation}
We note that non-negativity is required for a valid gradient descent to be obtained in \Cref{eq:euler_update}, and the bounded value is chosen to maintain stable training.

\textbf{Tangential Update.} To compute the tangential update $T_{V_{\mathcal{G}}}(\bfH^{(\ell)})$, we learn a dedicated GNN denoted by $\textsc{TangentGNN}$. Specifically, given input features $\bfH^{(\ell)}$, we predict a node feature update direction:
\begin{equation}
\label{eq:output_tan_MPNN}
\bfM^{(\ell)} = \sigma\left( \textsc{TangentGNN}(\bfH^{(\ell)}; \mathcal{G}) \right),
\end{equation}

and define the energy-tangential component via orthogonal projection, as described in \Cref{eq:tangential_flow}. Also, we define the scalar $\beta_{\mathcal{G}}$ that scales the tangential term, as follows:
\begin{equation}
    \label{eq:beta}
    \beta_{\cal{G}}(\bfH^{(l)}) = \text{MLP}_{{\beta}}\left(\textsc{SumPool}({\bfM}^{(\ell)})\right) \in \mathbb{R}.
\end{equation}

\section{Theoretical Properties of \ourmethod}
\label{sec:theoretical-analysis}

We now analyze the continuous-time dynamics of \ourmethod{} as defined in Equation~\eqref{eq:lyapunov_gnn_ode}. Our analysis focuses on three aspects: \emph{energy dissipation}, \emph{feature evolution in flat energy lanscapes}, and  \emph{the benefit of the tangent direction}. Proofs are provided in \Cref{app:proofs}.

\textbf{Assumptions and Notations.} Throughout this analysis, we assume that : (i) the input graph $\mathcal{G} = (\mathcal{V}, \mathcal{E})$ is connected; (ii) the energy function $V_{\mathcal{G}}(\bfH(t))$ is twice differentiable and bounded from below. For simplicity of notation, throughout this section we omit the time or layer scripts, and use the term $\bfH$ to denote node features, when possible.
   
We start by showing that \ourmethod{} is dissipative if $\|\nabla_{\bfH} V_{\mathcal{G}}(\bfH)\|^2  > 0$, and $\alpha_{\cal{G}}\geq 0$ (obtained by design),  corresponding to the Lyapunov stability criterion from \Cref{thm:lyapunov}.
\begin{proposition}[Energy Dissipation]
\label{prop:energy-dissipation}
Suppose $\alpha_{\mathcal{G}} \ge 0$ and $\|\nabla_{\bfH} V_{\mathcal{G}}(\bfH)\|^2  > 0$. Then the energy $V_{\mathcal{G}}(\bfH)$ is non-increasing along trajectories of Equation~\eqref{eq:lyapunov_gnn_ode}. Specifically,
\begin{eqnarray}
\frac{d}{dt} V_{\mathcal{G}}(\bfH) 
&=& -\alpha_{\mathcal{G}}(\bfH) \left\| \nabla_{\bfH} V_{\mathcal{G}}(\bfH) \right\|^2 
+ \beta_{\mathcal{G}}(\bfH) {\left\langle  T_{V_{\cal{G}}}(\bfH) , \nabla_{\bfH} V_{\mathcal{G}}(\bfH) \right\rangle} \nonumber \\
&=& -\alpha_{\mathcal{G}}(\bfH) \left\| \nabla_{\bfH} V_{\mathcal{G}}(\bfH) \right\|^2 
\le 0.
\label{eq:energy_dissipation}
\end{eqnarray}

\end{proposition}

We now show that unlike gradient flows, our \ourmethod{} admits evolution of node features in flat energy landscapes, a prime challenge in optimization techniques \citep{nw, boyd2004convex}.
\begin{proposition}[\ourmethod{} can Evolve Features in Flat Energy Landscapes]
    Suppose $\nabla_{\bfH} V_{\mathcal{G}}(\bfH) = 0$, and  $ T_{V_{\cal{G}}}(\bfH) \neq 0$, then the \ourmethod{} flow in \Cref{eq:lyapunov_gnn_ode} reads: $$\frac{d\bfH}{dt} = \beta_{\cal{G}}(\bfH) T_{V_{\cal{G}}}(\bfH).$$ This implies that in \emph{contrast} to gradient flows,  the dynamics of \ourmethod{} can evolve even in regions where the energy landscape is flat. 
\end{proposition}

\textbf{Theoretical Benefits of Using the Tangent Direction.}
Our \ourmethod{} combines two terms as shown in \Cref{eq:lyapunov_gnn_ode} and its discretization in \Cref{eq:euler_update}. These are the energy gradient $\nabla_{\bfH} V_{\mathcal{G}}(\bfH^{(\ell)})$ and the tangential direction vector $T_{V_{\cal{G}}}(\bfH)$. 
A natural theoretical and practical question is: \emph{under what conditions does the inclusion of the tangential direction improve over simple gradient descent?} To address this question, we begin by recalling a classic convergence result for gradient-based minimization.

\begin{proposition}[Convergence of Gradient Descent of a Scalar Function, \citet{nw}]
    Let $V_{\cal{G}}(\cdot)$ be a scalar function and let $\bfH^{(\ell+1)} = \bfH^{(\ell)} - \alpha_{\cal{G}}^{(\ell)}(\bfH^{(\ell)}){\nabla}_{\bfH} V_{\mathcal{G}}(\bfH^{(l)})$ be a gradient-descent iteration of the energy $V_{\cal{G}}(\cdot)$.
    Then, a linear convergence is obtained, with convergence rate:
    \[
    r = \frac{\lambda_{{\max}} - \lambda_{{\min}}}{\lambda_{{\max}} + \lambda_{{\min}}},
    \]
    where $\lambda_{{\max}}$ is the maximal eigenvalue, and in the case of problems that involve the graph Laplacian, $\lambda_{{\min}}$ is the second minimal eigenvalue, i.e., the first non-zero eigenvalue of the Hessian of $V_{\cal{G}}(\cdot)$.
    \label{prop:convergence_gradflow}
\end{proposition}

\Cref{prop:convergence_gradflow} shows that gradient descent suffers in ill-conditioned problems, i.e., when the ratio between the $\lambda_{\max}$ and $\lambda_{\min}$ is large. This is common in graph-based tasks, where the Hessian may inherit poor conditioning from the graph Laplacian, particularly when oversquashing occurs due to bottlenecks in the graph \cite{topping2022understanding, giraldo2022tradeoff, diGiovanniOversquashing}.

As an alternative, consider the effect of adding an orthogonal flow to the gradient descent direction. In this case, the combined update direction is 
\begin{equation}
\bfD = \alpha_{\cal{G}}(\bfH^{(\ell)}) \nabla_{\bfH} V_{\mathcal{G}}(\bfH^{(\ell)}) + \beta_{\cal{G}}(\bfH^{(\ell)})T_{V_{\mathcal{G}}}(\bfH^{(\ell)}).
\end{equation}

The following proposition demonstrates that it is possible to learn $T$ such that $\bfD$ becomes the Newton direction, which offers quadratic convergence \cite{nw}. 

\begin{proposition}[\ourmethod{} can learn a Quadratic Convergence Direction] 
\label{prop:accelerated_direction}
    Assume for simplicity that $\beta_{\cal{G}}=1$, and that the Hessian of $V_{\cal{G}}$ is invertible. Let $\bfD = \alpha_{\cal{G}}(\bfH^{(\ell)}) \nabla_{\bfH} V_{\mathcal{G}}(\bfH^{(\ell)}) + T_{V_{\mathcal{G}}}(\bfH^{(\ell)})$ with $\left\langle T_{V_{\mathcal{G}}}(\bfH^{(\ell)}), \widehat{\nabla}_{\bfH} V_{\mathcal{G}}(\bfH^{(\ell)}) \right\rangle = 0$. Then, it is possible to learn a direction $T_{V_{\mathcal{G}}}(\bfH^{(\ell)})$ and a step size $\alpha_{\cal{G}}$ such that $\bfD$ is the Newton direction, $\bfN = (\nabla^2V_{\cal{G}})^{-1} \nabla V_{\cal{G}}$.
\end{proposition}

In addition to its improved global convergence, Newton's method is notable for its local convergence rate behavior, being independent of the condition number of the Hessian \citep{nw, boyd2004convex}. This implies that if the tangential flow is learned to approximate Newton direction, \ourmethod{} can overcome the slow convergence caused by highly ill-conditioned energy landscapes, as commonly observed in different second order optimization techniques and their approximations, such as conjugate gradients (CG) and LBFGS \citep{nw,boyd2004convex}.
\emph{In the context of graph learning}, \Cref{prop:accelerated_direction} is particularly relevant when considering the oversquashing problem \citep{alon2021on,diGiovanniOversquashing}.  Oversquashing leads to poor conditioning; the graph Laplacian has a smallest eigenvalue of zero (for connected graphs), and the second smallest eigenvalue is also close to zero \cite{topping2022understanding,giraldo2022tradeoff, black2023understanding, jamadandi2024spectral}. Under these conditions, gradient flow methods, which are implicitly implemented by common GNN formulations \citep{digiovanni2023gradientflows}, perform poorly due to their ill-conditioned energy landscape, limiting the ability of propagating information between nodes. 

By enabling feature updates that can approximate second-order information, our \ourmethod{} provides a mechanism to actively mitigate oversquashing effects. 
We empirically validate our theoretical results in \Cref{fig:signal_comparison}, where we compare our method with a Dirichlet energy minimization process, which is often implemented by baseline GNNs \citep{rusch2023survey, digiovanni2023gradientflows}, with more details described in \Cref{app:experimental_details}, and further evaluate the effectiveness of our \ourmethod{} across oversquashing-related benchmarks in \Cref{sec:experiments}.

\begin{figure}[t]
    \centering
    \begin{subfigure}[t]{0.32\textwidth}
        \centering
    \includegraphics[width=\linewidth]{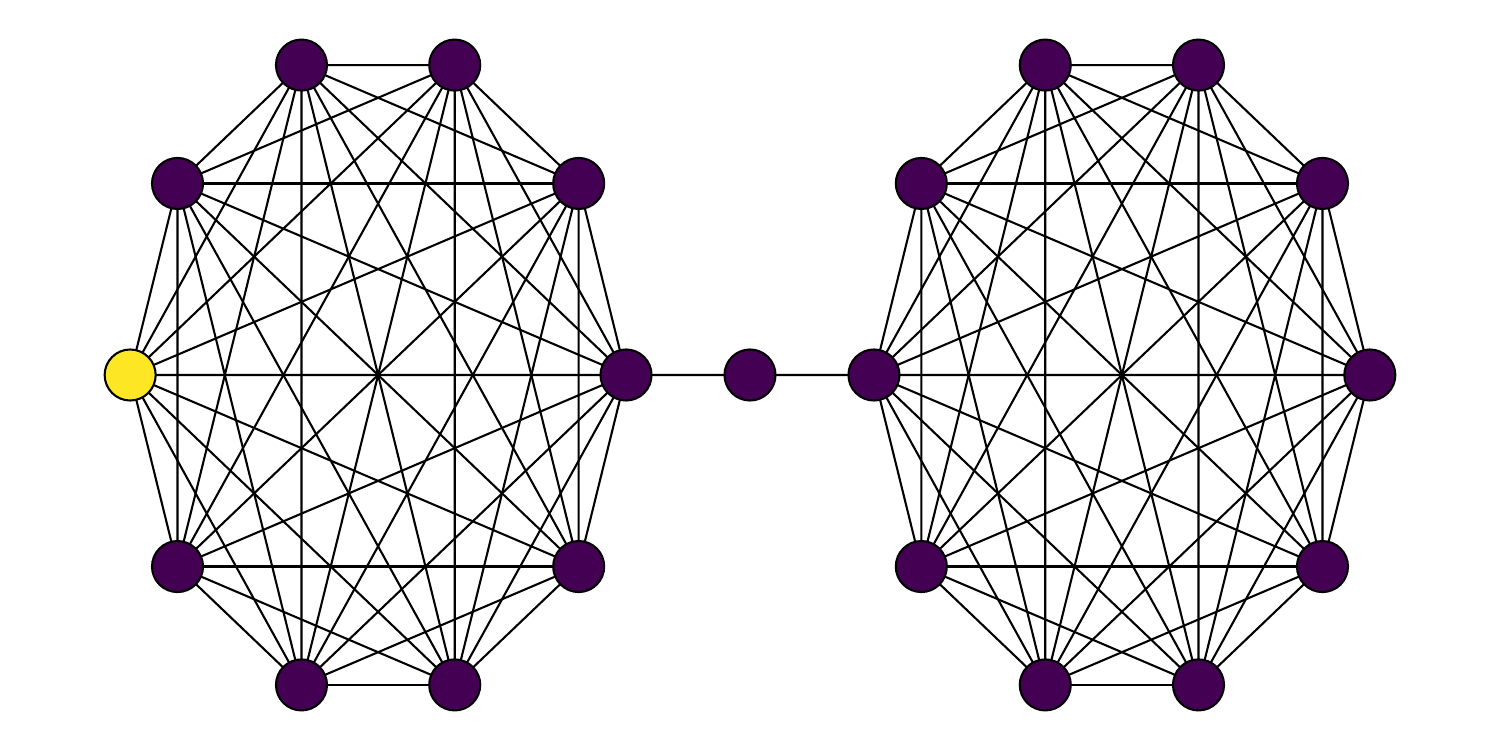}
        \caption{Initial Features}
        \label{fig:target}
    \end{subfigure}
    \hfill
    \begin{subfigure}[t]{0.32\textwidth}
        \centering
        \includegraphics[width=\linewidth]{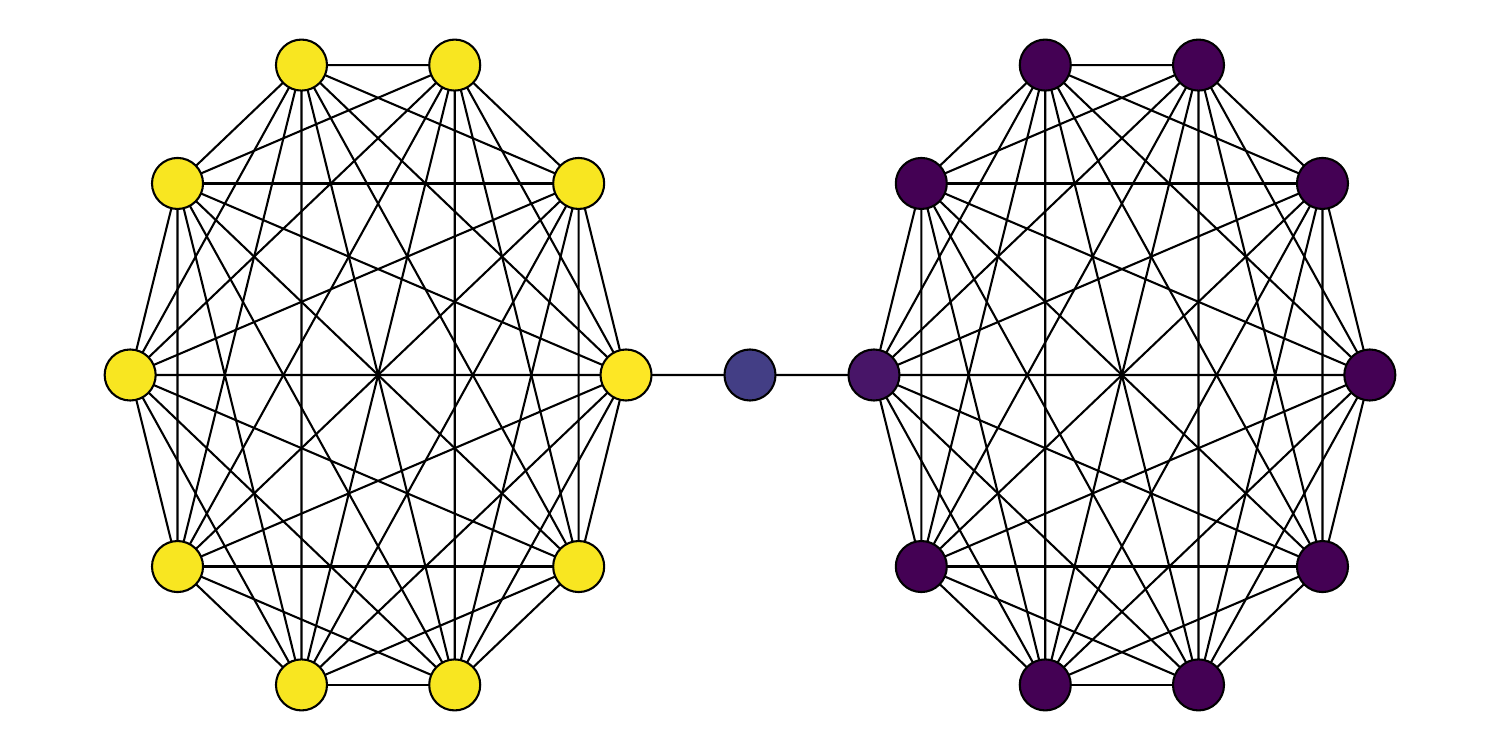}
        \caption{Gradient Flow}
        \label{fig:gradient}
    \end{subfigure}
    \hfill
    \begin{subfigure}[t]{0.32\textwidth}
        \centering
        \includegraphics[width=\linewidth]{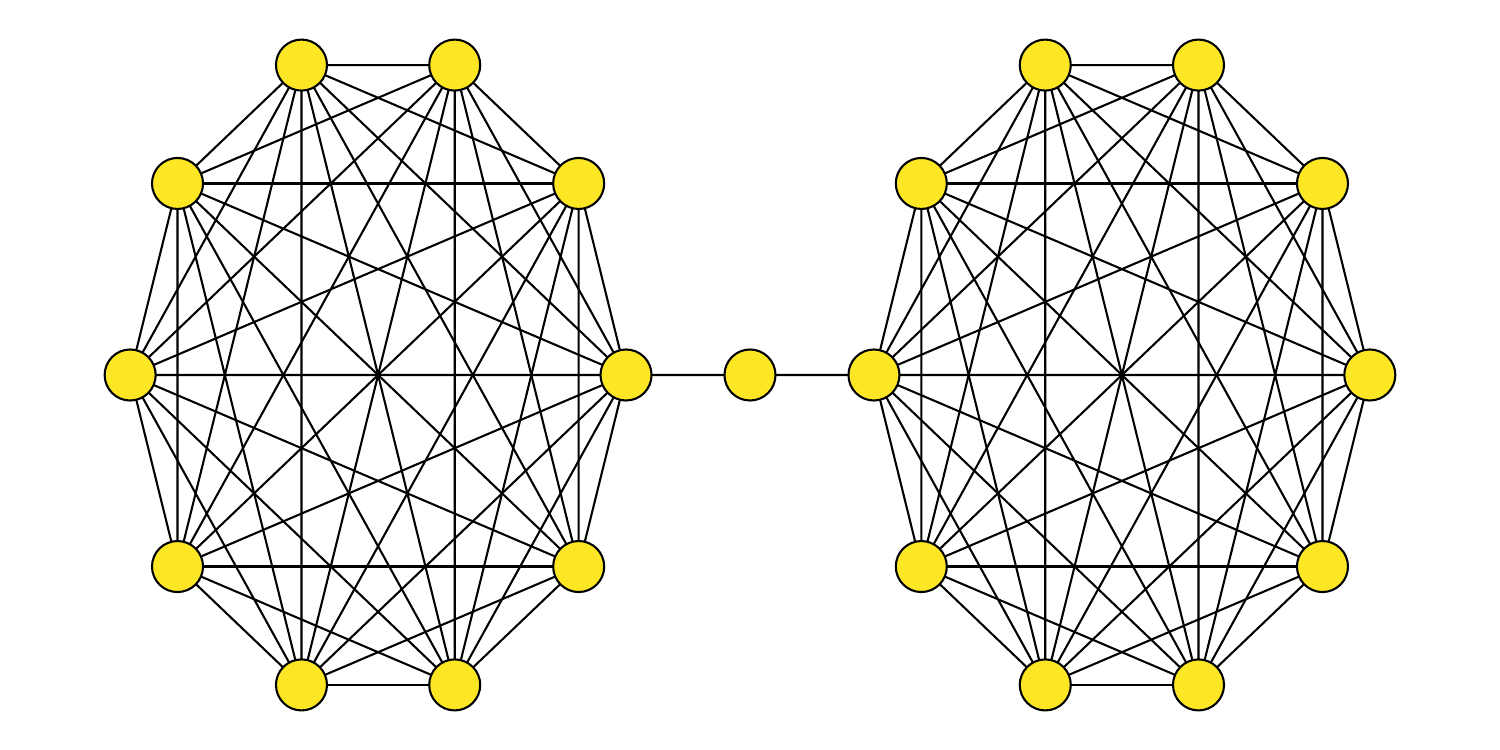}
        \caption{\ourmethod{}}
        \label{fig:lyapunov}
    \end{subfigure}
    \caption{Comparison of propagation behaviors between gradient flow and \ourmethod{} with 50 layers. While gradient flow struggles propagating information through the bottleneck, our \ourmethod{} is effective.}
    \label{fig:signal_comparison}
\end{figure}

\section{Experiments}
\label{sec:experiments}
We evaluate the performance of our \ourmethod{} on a suite of benchmarks: (i) synthetic benchmarks that require the exchange of messages with large distances, called graph property prediction from \citet{gravina_adgn}, in \Cref{sec:exp_gpp}; (ii) the peptides long-range graph benchmark \citep{LRGB} in Section~\ref{sec:exp_lrgb}; (iii) GNN benchmarks form \citep{dwivedi2023benchmarking} including the ZINC-12k, MNIST, CIFAR-10, PATTERN, and CLUSTER datasets; and (iv) the heterophilic node classification datasets from \citet{platonov2023a}. 
Notably, \ourmethod{} shows consistent downstream performance improvements over its baseline models, and it offers competitive performance compared with other popular and state-of-the-art methods, such as MPNN-based models, DE-GNNs, higher-order DGNs, and graph transformers. In \Cref{app:experimental_details} we provide experimental details on the hyperparameters, benchmark evaluation, and runtimes.  Additional results and comparisons, as well as evaluation on heterophilic node classification and an ablation study, are provided in \Cref{app:additional_results}.

\subsection{Graph Property Prediction}\label{sec:exp_gpp}

\begin{wraptable}{r}{0.525\textwidth}
\scriptsize
  \centering
  \vspace{-1.75em}  
  \setlength{\tabcolsep}{1pt}
\centering
\caption{Mean test set {\small$log_{10}(\mathrm{MSE})$}($\downarrow$) and std averaged on 4 random weight initializations on Graph Property Prediction. Lower is better. 
\one{First}, \two{second}, and \three{third} best results for each task are color-coded.
}
\label{tab:results_gpp}
\begin{tabular}{lccc}
\hline\toprule
\textbf{Model} &\textbf{Diameter} & \textbf{SSSP} & \textbf{Eccentricity} \\\midrule
\textbf{MPNNs} \\
$\,$ GatedGCN \citep{gatedgcn}& 0.1348$_{\pm 0.0397}$ & -3.2610$_{\pm 0.0514}$ & 0.6995$_{\pm 0.0302}$ \\
$\,$ GCN \citep{Kipf2016}            & 0.7424$_{\pm0.0466}$ & 0.9499$_{\pm0.0001}$ & 0.8468$_{\pm0.0028}$ \\
$\,$ GAT \citep{veličković2018graph}           & 0.8221$_{\pm0.0752}$ & 0.6951$_{\pm0.1499}$           & 0.7909$_{\pm0.0222}$  \\
$\,$ GraphSAGE \citep{SAGE}     & 0.8645$_{\pm0.0401}$ & 0.2863$_{\pm0.1843}$           &  0.7863$_{\pm0.0207}$\\
$\,$ GIN      \citep{xu2019how}      & 0.6131$_{\pm0.0990}$ & -0.5408$_{\pm0.4193}$          & 0.9504$_{\pm0.0007}$\\
$\,$  GCNII  \citep{gcnii}        & 0.5287$_{\pm0.0570}$ & -1.1329$_{\pm0.0135}$          & 0.7640$_{\pm0.0355}$\\
\midrule
\textbf{DE-GNNs} \\
$\,$ DGC  \citep{poli}          & 0.6028$_{\pm0.0050}$ & -0.1483$_{\pm0.0231}$          & 0.8261$_{\pm0.0032}$\\
$\,$ GRAND  \citep{chamberlain2021grand}        & 0.6715$_{\pm0.0490}$ & -0.0942$_{\pm0.3897}$          & 0.6602$_{\pm0.1393}$ \\
$\,$ GraphCON  \citep{rusch2022graph}     & 0.0964$_{\pm0.0620}$ & -1.3836$_{\pm0.0092}$ & 0.6833$_{\pm0.0074}$\\
$\,$ A-DGN \citep{gravina_adgn} 
& {-0.5188$_{\pm0.1812}$} & 
-3.2417$_{\pm0.0751}$ & {0.4296$_{\pm0.1003}$}  \\

$\,$ {SWAN} \citep{gravina_swan} & \two{-0.5981$_{\pm0.1145}$}  & {-3.5425$_{\pm0.0830}$}  & \three{-0.0739$_{\pm0.2190}$} \\
$\,$ {PH-DGN} \citep{gravina_phdgn}  & \three{-0.5385$_{\pm0.0187}$} &  \two{-4.2993$_{\pm0.0721 }$} & \two{-0.9348$_{\pm0.2097}$} \\
\midrule
\textbf{Transformers} \\
$\,$ GPS \citep{rampasek2022graphgps} & -0.5121$_{\pm0.0426}$ &  \three{-3.5990$_{\pm0.1949}$}  & 0.6077$_{\pm0.0282}$\\
\midrule
\textbf{Ours} \\
$\,$ \ourmethod$_{\textsc{GatedGCN}}$ &  {-0.6681$_{\pm0.0745}$} &  -5.0626$_{\pm0.0742}$ &  -1.7419$_{\pm0.0106}$
\\ 
$\,$ \ourmethod$_{\textsc{GPS}}$ & {\one{-0.9772$_{\pm0.0518}$}}   & \one{-5.5263$_{\pm0.0838}$} & \one{-2.1455$_{\pm0.0033}$}
\\ 

\bottomrule\hline      
\end{tabular}
  \vspace{-1em}
\end{wraptable}

\textbf{Setup.} We consider the three graph property prediction tasks from \citet{gravina_adgn}, 
evaluating the performance of \ourmethod{} in predicting graph diameters, single source shortest paths (SSSP), 
and node eccentricity on synthetic graphs. To effectively address these tasks, it is essential to propagate information not only from direct neighbors but also from distant nodes within the graph. As a result, strong performance in these tasks mirrors the ability to facilitate long-range interactions. 
\textbf{Results.}
Table \ref{tab:results_gpp} reports the mean test $log_{10}(\mathrm{MSE})$, comparing our \ourmethod{} with various MPNNs, DE-GNNs, and transformer-based models. The results highlight that \ourmethod{}, in all variants, consistently achieves the lowest (best) error across all tasks, demonstrating its efficacy compared with existing methods. For example, in the Eccentricity task, \ourmethod{}$_{\textsc{GPS}}$ reduces the error score by over 1.2 points compared to PH-DGN \citep{gravina_phdgn} and by over 2.0 points compared to SWAN, which are models designed to propagate information over long radii effectively. 
Overall, these results validate the effectiveness of our \ourmethod{} in modeling long-range interactions and mitigating oversquashing.
Furthermore, \ourmethod{} strengthens the performance of simple MPNN backbones like GatedGCN. For example, GatedGCN augmented with our \ourmethod{}  consistently delivers better results than the baseline GatedGCN, highlighting its ability to enhance traditional MPNNs. This demonstrates that our method can effectively leverage the strengths of simple models while overcoming their limitations in long-range propagation.

\subsection{GNN Benchmarking from \citet{dwivedi2023benchmarking}}\label{app:additional_benchmarks}
\textbf{Setup.} To further evaluate the performance of our \ourmethod{}, we consider multiple GNN from \citet{dwivedi2023benchmarking}, that include the \textit{ZINC-12k} dataset, MNIST and CIFAR-10 superpixels datasets, and CLUSTER and PATTERN datasets. These datasets are commonly used to evaluate state-of-the-art techniques \citep{ma2023graph}. For a fair and direct comparison with other methods, we follow the training and evaluation protocols from \citet{dwivedi2023benchmarking}.

\textbf{Results.} Table~\ref{tab:exp_bmgnn} reports the average and standard deviation of the obtained test metric. Besides ZINC-12k, which is a regression problem with mean absolute error (MAE) as the metric, all other datasets consider the accuracy(\%) metric. Our results show that across all benchmarks, our \ourmethod{} consistently improves its backbone performance, and often outperforms other  strong baselines.

\begin{table*}[t]
\setlength{\tabcolsep}{3pt}
\footnotesize
    \centering
    \caption{
    Test performance in five benchmarks from \cite{dwivedi2023benchmarking}. 
    Shown is the mean $_{\pm \text{std}}$ of 4 runs with different random seeds. Highlighted are the top \one{first}, \two{second}, and \three{third} results. }
    {\small
    \begin{tabular}{lccccc}
    \toprule
       \textbf{Model}  &  \textbf{ZINC-12k} & \textbf{MNIST} & \textbf{CIFAR10} & \textbf{PATTERN} & \textbf{ CLUSTER} \\
       \cmidrule{2-6} 
       & {MAE}$\downarrow$  &  {Accuracy}$\uparrow$ & {Accuracy}$\uparrow$ & {Accuracy}$\uparrow$ & {Accuracy}$\uparrow$ \\
       \midrule
GCN \citep{Kipf2016}  & 0.367$_{\pm 0.011}$ & 90.705$_{\pm 0.218}$ & 55.710$_{\pm 0.381}$ & 71.892$_{\pm 0.334}$ & 68.498$_{\pm 0.976}$ \\
GIN \citep{xu2019how}  & 0.526$_{\pm 0.051}$ & 96.485$_{\pm 0.252}$ & 55.255$_{\pm 1.527}$ & 85.387$_{\pm 0.136}$ & 64.716$_{\pm 1.553}$ \\
GAT \citep{veličković2018graph} & 0.384$_{\pm 0.007}$ & 95.535$_{\pm 0.205}$ & 64.223$_{\pm 0.455}$ & 78.271$_{\pm 0.186}$ & 70.587$_{\pm 0.447}$ \\
GatedGCN \citep{gatedgcn} & 0.282$_{\pm 0.015}$ & 97.340$_{\pm 0.143}$ & 67.312$_{\pm 0.311}$ & 85.568$_{\pm 0.088}$ & 73.840$_{\pm 0.326}$ \\
PNA \citep{PNA} & 0.188$_{\pm 0.004}$ & 97.94$_{\pm 0.12}$ & 70.35$_{\pm 0.63}$ & $-$ & $-$ \\
DGN  \citep{beaini2021directional} & 0.168$_{\pm 0.003}$ & $-$ & \three{72.838$_{\pm 0.417}$} & 86.680$_{\pm 0.034}$ & $-$ \\
\midrule 
CRaW1 \citep{tonnshoff2023walking} & 0.085$_{\pm 0.004}$ & 97.944$_{\pm 0.050}$ & 69.013$_{\pm 0.259}$ & $-$ & $-$ \\
GIN-AK+ \citep{zhao2022from} & 0.080$_{\pm 0.001}$ & $-$ & 72.19$_{\pm 0.13}$ & \three{86.850$_{\pm 0.057}$} & $-$ \\
\midrule SAN \citep{kreuzer2021rethinking} & 0.139$_{\pm 0.006}$ & $-$ & $-$ & 86.581$_{\pm 0.037}$ & 76.691$_{\pm 0.65}$ \\
EGT \citep{hussain2022global} & 0.108$_{\pm 0.009}$ & \two{98.173$_{\pm 0.087}$} & 68.702$_{\pm 0.409}$ & 86.821$_{\pm 0.020}$ & \three{79.232$_{\pm 0.348}$} \\
Graphormer-GD \citep{zhang2023rethinking} & 0.081$_{\pm 0.009}$ & $-$ & $-$ & $-$ & $-$ \\
GPS \citep{rampasek2022graphgps} & \three{0.070$_{\pm 0.004}$} & {98.051$_{\pm 0.126}$} & {72.298$_{\pm 0.356}$} & 86.685$_{\pm 0.059}$ & {78.016$_{\pm 0.180}$} \\
GRIT \citep{ma2023graph} & \one{0.059$_{\pm 0.002}$} & \three{98.108$_{\pm 0.111}$} & \one{76.468$_{\pm 0.881}$} & \one{87.196$_{\pm 0.076}$} & \two{80.026$_{\pm 0.277}$} \\
\midrule
\ourmethod{}\textsubscript{GatedGCN} & 0.128$_{\pm 0.011}$ & 97.788$_{\pm 0.105}$ & 70.894$_{\pm 0.329}$ & 86.672$_{\pm 0.071}$ & 78.194$_{\pm 0.307}$ \\
\ourmethod{}\textsubscript{GPS} & \two{0.062$_{\pm 0.005}$} & \one{98.197$_{\pm 0.110}$} & \two{75.783$_{\pm 0.261}$} &  \two{87.182$_{\pm 0.063}$} & \one{80.113$_{\pm 0.138}$} \\
       \bottomrule
    \end{tabular}
    }\\
    \label{tab:exp_bmgnn}
\end{table*}

\begin{wraptable}{r}{0.45\textwidth}
  \centering
  \scriptsize
    \setlength{\tabcolsep}{2pt}

  \vspace{-2em}  
    \centering
    \caption{Results for Peptides-func and Peptides-struct (3 training seeds).  The \one{first}, \two{second}, and \three{third} best scores are colored.
    }\label{tab:results_lrgb}
    
    \begin{tabular}{@{}lcc@{}}
    \hline\toprule
    \multirow{2}{*}{\textbf{Model}} & \textbf{Peptides-func}  & \textbf{Peptides-struct}              
    \\
   
    & \scriptsize{AP $\uparrow$} & \scriptsize{MAE $\downarrow$}                            
    \\ \midrule  
    \textbf{MPNNs} \\
    $\,$ GCN \citep{Kipf2016}& 59.30$_{\pm0.23}$ & 0.3496$_{\pm0.0013}$ \\ 
    $\,$ GINE \citep{dwivedi2023benchmarking}                      & 54.98$_{\pm0.79}$ & 0.3547$_{\pm0.0045}$ \\
    $\,$ GCNII \citep{gcnii}                     & 55.43$_{\pm0.78}$ & 0.3471$_{\pm0.0010}$ \\ 
    $\,$ GatedGCN   \citep{gatedgcn}                & 58.64$_{\pm0.77}$ & 0.3420$_{\pm0.0013}$ \\ 
    \midrule
    \textbf{Multi-hop GNNs}\\
    $\,$ DIGL+MPNN+LapPE  \citep{DIGL}    & 68.30$_{\pm0.26}$         & 0.2616$_{\pm0.0018}$ \\ 
    $\,$ MixHop-GCN+LapPE \citep{abu2019mixhop}    & 68.43$_{\pm0.49}$         & 0.2614$_{\pm0.0023}$ \\ 
    $\,$ DRew-GCN+LapPE \citep{drew}            & \one{71.50$_{\pm0.44}$}   & 0.2536$_{\pm0.0015}$ \\ 
    \midrule
    
    \textbf{Transformers} \\
    $\,$ Transformer+LapPE \citep{dwivedi2023benchmarking} & 63.26$_{\pm1.26}$ & 0.2529$_{\pm0.0016}$           \\ 
    $\,$ SAN+LapPE  \citep{kreuzer2021rethinking}        & {63.84$_{\pm1.21}$} & 0.2683$_{\pm0.0043}$           \\
    $\,$ GPS+LapPE \citep{rampasek2022graphgps}   & 65.35$_{\pm0.41}$ & 0.2500$_{\pm0.0005}$   \\ 
    \midrule
    \textbf{DE-GNNs} \\
    $\,$ GRAND    \citep{chamberlain2021grand}  & 57.89$_{\pm0.62}$ & 0.3418$_{\pm0.0015}$ \\ 
    $\,$ GraphCON \citep{rusch2022graph}  & 60.22$_{\pm0.68}$ & 0.2778$_{\pm0.0018}$ \\
    $\,$ A-DGN   \citep{gravina_adgn}   & 59.75$_{\pm0.44}$ & 0.2874$_{\pm0.0021}$ \\
    $\,$ {SWAN} \citep{gravina_swan}& 67.51$_{\pm0.39}$ & \three{0.2485$_{\pm0.0009}$}\\
    $\,$ {PH-DGN} \citep{gravina_phdgn}
    & \three{70.12$_{\pm0.45}$} & \two{0.2465$_{\pm0.0020}$}\\
    \midrule
    \textbf{Ours} \\
    $\,$ \ourmethod$_{\textsc{GatedGCN}}$ & 68.92$_{\pm0.40}$ & 0.2451$_{\pm0.0006}$  \\
        $\,$ \ourmethod$_{\textsc{GPS}}$ & \two{70.21$_{\pm0.43}$}  & \one{0.2422$_{\pm0.0014}$}  \\

    \bottomrule\hline
    \end{tabular}
  \vspace{-5em}
\end{wraptable}

\subsection{Long-Range Benchmark}
\label{sec:exp_lrgb}

\textbf{Setup.} We assess the performance of our method on the real-world long-range graph benchmark (LRGB) from \cite{LRGB}, focusing on the \textit{Peptides-func} and \textit{Peptides-struct} datasets. 
We follow the experimental setting in \cite{LRGB}, including the 500K parameter budget. All transformer baselines include positional and structural encodings.  \ourmethod{} does not use additional encodings. The datasets consist of large molecular graphs derived from peptides, 
where the structure and function of a peptide depend on interactions between distant parts of the graph. Therefore, relying on short-range interactions, such as those captured by local message passing in GNNs, 
may not be sufficient to excel at this task. 

\textbf{Results.} Table \ref{tab:results_lrgb} provides a comparison of our \ourmethod{} model with a wide range of baselines. A broader comparison 
is presented in Table \ref{tab:results_lrgb_complete}.
The results indicate that \ourmethod{} outperforms standard MPNNs, transformer-based GNNs, DE-GNNs, and most Multi-hop GNNs.

\subsection{Heterophilic Node Classification}
\label{sec:heterophilic}
\textbf{Setup.} 
We consider heterophilic node classification datasets;  \textit{Roman-empire, \textit{Amazon-ratings}, \textit{Minesweeper}, \textit{Tolokers}, and \textit{Questions}} tasks, 
to evaluate \ourmethod{} in capturing 
complex node relationships beyond simple homophily. We follow the training and evaluation protocols from \citet{platonov2023a}.

 \textbf{Results.} We report the performance of \ourmethod{} in   \Cref{app:heterophilic_results}, and compare it with several recent leading methods.  Specifically, we include baseline results from \cite{finkelshtein2024cooperative, platonov2023a, muller2024attending}.
Across all datasets, \ourmethod{} achieves competitive performance that often outperforms state-of-the-art methods, demonstrating that our \ourmethod{} can also be utilized on larger graphs and in complex heterophilic scenarios.

\section{Related Work}
\label{sec:related}
In this section we cover two main topics related to our \ourmethod{}. In \Cref{app:additional_related} we discuss additional related works.

\textbf{Deep GNNs and Dynamical Systems.} 
A growing body of work interprets GNN layers as iterative updates in a dynamical system, providing a principled framework to analyze stability, control diffusion, and inform architectural design. \citet{poli2019gnnode} introduced Graph Neural ODEs, inspired by neural ODEs~\citep{RuthottoHaber2018, chen2018neural}, modeling node feature evolution via continuous-depth ODEs aligned with graph structure, enabling adaptive computation and improved performance in dynamic settings. Similarly, \citet{xhonneux2020continuousgnn} proposed Continuous GNNs, where feature channels evolve by differential equations, mitigating over-smoothing via infinite-depth limits. Follow-up works such as GODE~\citep{zhuang2020ordinary}, GRAND~\citep{chamberlain2021grand}, PDE-GCN\textsubscript{D}~\citep{eliasof2021pde}, and DGC~\cite{DGC} view GNN layers as discrete integration steps of the heat equation to control oversmoothing~\citep{nt2019revisiting, oono2020graph, cai2020note}. Extensions like PDE-GCN\textsubscript{M}~\citep{eliasof2021pde} and GraphCON~\citep{rusch2022graph} add oscillatory components to preserve feature energy, while others leverage heat-kernel attention~\citep{pmlr-v162-choromanski22a}, anti-symmetry~\citep{gravina_adgn, gravina_swan}, reaction-diffusion~\citep{wang2022acmp, choi2022gread}, advection-reaction-diffusion~\citep{eliasof2024feature} to enhance long-range or directional flow, and higher-order graph neuro ODE models \citep{eliasof2024temporal}. A comprehensive overview is given in \citet{han2023continuous}. Closely related, \citet{digiovanni2023gradientflows} interpret GNN layer updates as gradient flows of the Dirichlet energy, aligning message passing with energy minimization. In contrast, our \ourmethod{} learns a graph-adaptive, task-specific energy and introduces a novel descent mechanism combining energy gradients with a learnable tangential component, enabling more expressive dynamics than pure gradient flows.

\textbf{Learning Energy Functions in Neural Networks.} Energy-based models (EBMs) provide a flexible framework in deep learning by learning an energy function whose low-energy regions correspond to 
areas with high probability for the data.
They have been widely used in generative tasks such as image synthesis~\citep{lecun2006tutorial, xie2016theory, du2019implicit, guo2023egc} and graph generation~\citep{liu2021graphebm, reiser2022graph}. In contrast to these typically unsupervised settings, our work focuses on learning a \emph{task-driven} energy function tailored to predictive objectives like node or graph classification. Here, inference corresponds to descending the learned energy landscape, whose minima align with correct outputs. Relatedly, Lyapunov functions—classical tools from control theory—have been used in neural networks to ensure stable learning or inference dynamics, e.g., by enforcing stability in Neural ODEs~\citep{rodriguez2022lyanet} or GNN-based controllers~\citep{fallin2025lyapunov}. However, such approaches typically assume a fixed or implicit energy function rather than learning one. Our method, \ourmethod{}, bridges and extends these perspectives by learning a graph-adaptive, task-specific energy and introducing a novel optimization scheme. Crucially, our \ourmethod{} incorporates a learnable tangential component that accelerates energy minimization and enhances performance in graph learning tasks.

\section{Conclusions}
\label{sec:conclusions}
We introduced \ourmethod{}, a novel framework for learning graph neural dynamics through the joint modeling of an energy descent direction and a tangential flow. By interpreting GNN message passing through the lens of Lyapunov theory and continuous dynamical systems, \ourmethod{} unifies task-driven energy-based modeling with flexible, learnable tangential flows, which allow for better utilization of the learned energy function by accelerating its minimization. We further show that the tangential component enables continued feature evolution in flat or ill-conditioned energy landscapes, offering a compelling advantage over traditional gradient flow approaches. We relate this property to the mitigation of oversquashing, a persistent challenge in graph learning. 
Empirically, \ourmethod{} achieves strong performance across 15 synthetic and real-world benchmarks, outperforming message-passing, diffusion-based, and attention-based GNNs.
This work opens several interesting directions for future research, including the incorporation of higher-order differential operators into the tangential flow mechanism, and an analysis and regularization techniques for the learned energy landscape.

\newpage
\clearpage

\bibliography{biblio}
\bibliographystyle{plainnat}

\newpage
\clearpage

\clearpage
\appendix

\section{Additional Related Work}
\label{app:additional_related}

\textbf{Oversquashing in Graph Learning.}
Graph neural networks (GNNs) typically operate through message-passing mechanisms, aggregating information from local neighborhoods. While effective in capturing short-range dependencies, this design often leads to \emph{oversquashing}, a phenomenon where signals from distant nodes are compressed into fixed-size representations, impeding the flow of long-range information \citep{alon2021on, diGiovanniOversquashing, topping2022understanding}. This limitation poses a challenge in domains that demand rich global context, such as bioinformatics \citep{baek2021accurate, LRGB} and heterophilic graphs \citep{luan2024heterophilicgraphlearninghandbook, snowflake}. 
A range of strategies have been proposed to mitigate oversquashing. \emph{Graph rewiring} approaches, such as SDRF \citep{topping2022understanding}, densify the graph to enhance connectivity prior to training. In contrast, methods like GRAND \citep{chamberlain2021grand}, BLEND \citep{blend}, and DRew \citep{drew} adjust the graph structure dynamically based on node features. \emph{Transformer-based models} offer another promising route by leveraging global attention to enable direct, long-range message passing. Examples include SAN \citep{Kreuzer2021}, Graphormer \citep{Ying2021}, and GPS \citep{graphgps}, which incorporate positional encodings, such as Laplacian eigenvectors \citep{dwivedi2023benchmarking} and random walk structural embeddings \citep{dwivedi2022graph} to preserve structural identity. However, the quadratic complexity of full attention in these models raises scalability concerns, motivating interest in sparse attention mechanisms \citep{Zaheer2020, Choromanski2020, Shirzad2023}. 
An alternative line of work explores \emph{non-local dynamics} to enhance expressivity without relying solely on attention. FLODE \citep{maskey2023fractional} employs fractional graph operators, QDC \citep{markovich2023qdc} uses quantum diffusion processes, and G2TN \citep{g2tn} models explicit diffusion paths to propagate information more effectively. While these approaches address the oversquashing bottleneck, they often come with increased computational demands due to dense propagation operators. For a broader overview of these techniques, see \citet{shi2023exposition}. We note that the challenge of modeling long-range dependencies also arises in other domains, such as sequential architectures  \citep{lstm,ssm}.

\paragraph{Optimization Techniques.}
The formulation of \ourmethod{} draws parallel with concepts that have been explored in the optimization literature, particularly in the design of dynamical systems that balance expressivity and convergence. While traditional gradient descent provides a robust and interpretable mechanism for minimizing energy functions, its convergence rate can be limited in poorly conditioned settings \citep{boyd2004convex, nw}, which frequently arise in graph-based problems due to structural bottlenecks \citep{alon2021on, topping2022understanding}. Second-order approaches, such as Newton's method, are known to accelerate convergence by incorporating curvature information, albeit at increased computational cost. 
The combination of energy gradient descent and a learned tangential component in \ourmethod{} suggests a learnable departure from purely first-order schemes. Rather than explicitly computing or approximating the Hessian, our framework enables the model to learn corrective update directions that are orthogonal to the descent path. This design implicitly aligns with the motivations behind quasi-Newton techniques like conjugate gradients and LBFGS \citep{nw}, which aim to improve convergence by leveraging directional information that complements the gradient. 
From this perspective, \ourmethod{} can be viewed as embedding optimization-inspired dynamics within graph learning frameworks. This is particularly relevant in scenarios affected by oversquashing \citep{diGiovanniOversquashing}, where effective feature transmission often requires departing from strictly local, gradient-driven updates. By allowing energy-preserving tangential flows, \ourmethod{} introduces flexibility reminiscent of structured optimization methods, adapted to the graph learning domain.

\section{Proofs of Theoretical Results}
\label{app:proofs}
\setcounter{proposition}{0}

In this section, we restate the theoretical results from \Cref{sec:theoretical-analysis} and provide their proofs. As in the main text, we assume the following throughout: (i) the input graph $\mathcal{G} = (\mathcal{V}, \mathcal{E})$ is connected; (ii) the energy function $V_{\mathcal{G}}(\bfH(t))$ is twice differentiable and bounded from below. For simplicity of notation, throughout this section, we omit the time or layer scripts and use the term $\bfH$ to denote node features when possible.

\begin{proposition}[Energy Dissipation]
\label{prop:energy-dissipation_app}
Suppose $\alpha_{\mathcal{G}} \ge 0$ and $\|\nabla_{\bfH} V_{\mathcal{G}}(\bfH)\|^2  > 0$. Then the energy $V_{\mathcal{G}}(\bfH)$ is non-increasing along trajectories of Equation~\eqref{eq:lyapunov_gnn_ode}. Specifically,
\begin{eqnarray}\nonumber
\frac{d}{dt} V_{\mathcal{G}}(\bfH) 
&=& -\alpha_{\mathcal{G}}(\bfH) \left\| \nabla_{\bfH} V_{\mathcal{G}}(\bfH) \right\|^2 
+ \beta_{\mathcal{G}}(\bfH) {\left\langle  T_{V_{\cal{G}}}(\bfH) , \nabla_{\bfH} V_{\mathcal{G}}(\bfH) \right\rangle} \nonumber \\ \nonumber
&=& -\alpha_{\mathcal{G}}(\bfH) \left\| \nabla_{\bfH} V_{\mathcal{G}}(\bfH) \right\|^2 
\le 0.
\label{eq:energy_dissipation_app}
\end{eqnarray}

\end{proposition}

\begin{proof}
By the chain rule,
\[
\frac{d}{dt} V_{\mathcal{G}}(\bfH) = \left\langle \nabla_{\bfH} V_{\mathcal{G}}(\bfH),\, \frac{d\bfH}{dt} \right\rangle.
\]
Substituting the dynamics of Equation~\eqref{eq:lyapunov_gnn_ode}:
\begin{eqnarray}
\nonumber
\frac{d}{dt} V_{\mathcal{G}}(\bfH) &=& \left\langle \nabla_{\bfH} V_{\mathcal{G}}(\bfH),\, -\alpha_{\mathcal{G}}(\bfH)\, \nabla_{\bfH} V_{\mathcal{G}}(\bfH) + \beta_{\mathcal{G}}(\bfH)\, T_{V_{\mathcal{G}}}(\bfH) \right\rangle \\
\nonumber
&=& -\alpha_{\mathcal{G}}(\bfH)\, \left\| \nabla_{\bfH} V_{\mathcal{G}}(\bfH) \right\|^2 + \beta_{\mathcal{G}}(\bfH)\, \left\langle T_{V_{\mathcal{G}}}(\bfH),\, \nabla_{\bfH} V_{\mathcal{G}}(\bfH) \right\rangle.
\end{eqnarray}
As discussed in \Cref{sec:method}, we have by design, that 
\[
\left\langle T_{V_{\mathcal{G}}}(\bfH),\, \nabla_{\bfH} V_{\mathcal{G}}(\bfH) \right\rangle = 0.
\]
Therefore,
\[
\frac{d}{dt} V_{\mathcal{G}}(\bfH) = -\alpha_{\mathcal{G}}(\bfH)\, \left\| \nabla_{\bfH} V_{\mathcal{G}}(\bfH) \right\|^2.
\]
Because $\alpha_{\mathcal{G}}(\bfH) \ge 0$ by design, the energy is non-increasing, and assuming $\alpha_{\mathcal{G}}(\bfH) > 0$, the system is dissipative, i.e., its energy is decreasing.
\end{proof}

\begin{proposition}[\ourmethod{} can Evolve Features in Flat Energy Landscapes]
    Suppose $\nabla_{\bfH} V_{\mathcal{G}}(\bfH) = 0$, and  $ T_{V_{\cal{G}}}(\bfH) \neq 0$, then the \ourmethod{} flow in \Cref{eq:lyapunov_gnn_ode} reads: $$\frac{d\bfH}{dt} = \beta_{\cal{G}}(\bfH) T_{V_{\cal{G}}}(\bfH).$$ This implies that in \emph{contrast} to gradient flows,  the dynamics of \ourmethod{} can evolve even in regions where the energy landscape is flat. 
\end{proposition}

\begin{proof}
Because $\nabla_{\bfH} V_{\mathcal{G}}(\bfH) = 0$, the first term in \Cref{eq:lyapunov_gnn_ode} vanishes, and the \ourmethod{} dynamical system reads:
\[
\frac{d\bfH}{dt}=\beta_{\mathcal{G}}(\bfH) T_{V_{\cal{G}}}(\bfH),
\]
Assuming that $ T_{V_{\cal{G}}}(\bfH) \neq 0$, \ourmethod{} can continue evolving node features also in cases where $\nabla_{\bfH} V_{\mathcal{G}}(\bfH) = 0$, i.e., where the energy landscape is flat.
\end{proof}

\begin{proposition}[Convergence of Gradient Descent of a Scalar Function, \citet{nw}]
    Let $V_{\cal{G}}(\cdot)$ be a scalar function and let $\bfH^{(\ell+1)} = \bfH^{(\ell)} - \alpha_{\cal{G}}^{(\ell)}(\bfH^{(\ell)}){\nabla}_{\bfH} V_{\mathcal{G}}(\bfH^{(l)})$ be a gradient-descent iteration of the energy $V_{\cal{G}}(\cdot)$.
    Then, a linear convergence is obtained, with convergence rate:
    \[
    r = \frac{\lambda_{{\max}} - \lambda_{{\min}}}{\lambda_{{\max}} + \lambda_{{\min}}},
    \]
    where $\lambda_{{\max}}$ is the maximal eigenvalue, and in the case of problems that involve the graph Laplacian, $\lambda_{{\min}}$ is the second minimal eigenvalue, i.e., the first non-zero eigenvalue of the Hessian of $V_{\cal{G}}(\cdot)$.
    \label{prop:convergence_gradflow_app}
\end{proposition}

\begin{proposition}[\ourmethod{} can learn a Quadratic Convergence Direction] 
\label{prop:accelerated_direction_app}
    Assume for simplicity that $\beta_{\cal{G}}=1$, and that the Hessian of $V_{\cal{G}}$ is invertible. Let $\bfD = \alpha_{\cal{G}}(\bfH^{(\ell)}) \nabla_{\bfH} V_{\mathcal{G}}(\bfH^{(\ell)}) + T_{V_{\mathcal{G}}}(\bfH^{(\ell)})$ with $\left\langle T_{V_{\mathcal{G}}}(\bfH^{(\ell)}), \widehat{\nabla}_{\bfH} V_{\mathcal{G}}(\bfH^{(\ell)}) \right\rangle = 0$. Then, it is possible to learn a direction $T_{V_{\mathcal{G}}}(\bfH^{(\ell)})$ and a step size $\alpha_{\cal{G}}$ such that $\bfD$ is the Newton direction, $\bfN = (\nabla^2V_{\cal{G}})^{-1} \nabla V_{\cal{G}}$.
\end{proposition}

\begin{proof}
We aim to construct a direction $\bfD = \alpha_{\mathcal{G}}(\bfH)\, \nabla_{\bfH} V_{\mathcal{G}}(\bfH) + T_{V_{\mathcal{G}}}(\bfH)$ that matches the Newton direction:
\[
\bfN = \left( \nabla^2_{\bfH} V_{\mathcal{G}}(\bfH) \right)^{-1} \nabla_{\bfH} V_{\mathcal{G}}(\bfH).
\]
Recall that by design, we have  that $T_{V_{\mathcal{G}}}(\bfH)$ is orthogonal to the energy gradient, i.e., $
\left\langle T_{V_{\mathcal{G}}}(\bfH),\, \nabla_{\bfH} V_{\mathcal{G}}(\bfH) \right\rangle = 0.$ 
Then, we can express a Newton direction by the decomposition:
\[
\bfN = \alpha_{\mathcal{G}}(\bfH)\, \nabla_{\bfH} V_{\mathcal{G}}(\bfH) + T_{V_{\mathcal{G}}}(\bfH).
\]
Solving for the orthogonal component yields:
\[
T_{V_{\mathcal{G}}}(\bfH) = \bfN - \alpha_{\mathcal{G}}(\bfH)\, \nabla_{\bfH} V_{\mathcal{G}}(\bfH).
\]
To enforce orthogonality, we require:
\[
\left\langle \bfN - \alpha_{\mathcal{G}}(\bfH)\, \nabla_{\bfH} V_{\mathcal{G}}(\bfH),\, \nabla_{\bfH} V_{\mathcal{G}}(\bfH) \right\rangle = 0.
\]
Expanding and simplifying, we find:
\[
\left\langle \bfN,\, \nabla_{\bfH} V_{\mathcal{G}}(\bfH) \right\rangle - \alpha_{\mathcal{G}}(\bfH)\, \left\| \nabla_{\bfH} V_{\mathcal{G}}(\bfH) \right\|^2 = 0,
\]
and the optimal step size is given by:
\[
\alpha_{\mathcal{G}}(\bfH) = \frac{\left\langle \bfN,\, \nabla_{\bfH} V_{\mathcal{G}}(\bfH) \right\rangle}{\left\| \nabla_{\bfH} V_{\mathcal{G}}(\bfH) \right\|^2},
\]
showing that it is possible to learn a Newton direction, i.e., a quadratic energy convergence direction.
\end{proof}


\section{Complexity and Runtimes}
\label{app:runtimes_complexity}

\textbf{Complexity.}
Each step of \ourmethod{} requires computing the gradient of the learned energy function $V_{\mathcal{G}}(\bfH^{(\ell)})$, that is defined in \Cref{eq:energy_NN}. This involves two main operations: (i) forward and backward passes through the energy network $\textsc{EnergyGNN}$, which contains $L_{\text{energy}}$ message-passing layers and an MLP; and (ii) automatic differentiation to compute $\nabla_{\bfH} V_{\mathcal{G}}(\bfH^{(\ell)})$ with respect to the input node features. In parallel, the tangential flow direction $T_{V_{\mathcal{G}}}(\bfH^{(\ell)})$ is obtained by projecting the vector field $\bfM^{(\ell)}$ computed by a separate $\textsc{TangentGNN}$ with $L_{\text{tangent}}$ layers onto the orthogonal complement of the normalized energy gradient, as shown in \Cref{eq:tangential_flow}. This projection is of computational cost of $O(n d)$ per step, where $n = |\mathcal{V}|$ and $d$ is the feature dimensionality. 
In addition, scalar coefficients $\alpha_{\mathcal{G}}$ and $\beta_{\mathcal{G}}$ are computed from pooled node features using MLPs (\Cref{eq:alpha,eq:beta}). Assuming both $\textsc{EnergyGNN}$ and $\textsc{TangentGNN}$ are message-passing architectures with linear complexity in the number of nodes and edges, and setting $L_{\text{energy}} = L_{\text{tangent}}$, the total complexity per layer becomes $O(L_{\text{gnn}} \cdot (n + m) \cdot d)$, where $L_{\text{gnn}}$ is the number of GNN layers used in each subnetwork and $m = |\mathcal{E}|$ is the number of edges. 
Unrolling the dynamics over $L$ steps, the overall computational complexity of \ourmethod{} is:
\[
O\left(L \cdot L_{\text{gnn}} \cdot (|\mathcal{V}| + |\mathcal{E}|) \cdot d \right).
\]

\textbf{Runtimes.}
We measure the runtimes of our \ourmethod{} using two backbones, GatedGCN and GPS, and compare it with the baseline backbone runtimes. In addition, we consider other methods like FAGCN \citep{FAGCN} and CO-GNN \citep{finkelshtein2024cooperative} for a broad comparison of the runtimes of \ourmethod{}. For reference, we also refer to \Cref{tab:results_hetero}, where we compare the obtained downstream performance, which shows in many cases significant improvement using our \ourmethod{} variants compared with other considered methods. We measure the runtimes on the Questions dataset, using the same major hyperparameters for all methods (256 channels, 8 layers) to ensure fairness. The measurements were conducted on an NVIDIA RTX6000 Ada GPU with 48GB of memory.  

\begin{table}[ht]
\scriptsize
\setlength{\tabcolsep}{4pt}
    \centering
    \caption{Training runtimes (milliseconds per epoch) on the Questions dataset using an 8-layer network with 256 channels on an NVIDIA RTX6000 Ada GPU.}
    \begin{tabular}{l|cccccccc}
        \toprule
        Method & GCN & CO-GNN & FAGCN & GatedGCN & GAT & GPS(GatedGCN)  &  \ourmethod{}\textsubscript{GatedGCN} & \ourmethod{}\textsubscript{GPS} \\
        \midrule
        Runtime (ms/epoch) & 69.77 & 210.32 & 103.94 & 129.92 & 112.40 & 429.08  & 184.98 & 694.27   \\
        
        \bottomrule
    \end{tabular}
    \label{tab:timings2}
\end{table}

\section{Experimental Details}\label{app:experimental_details}
In this section, we provide additional experimental details. 

\textbf{Computational Resources.} Our experiments are run on NVIDIA RTX6000 Ada with 48GB of memory. Our code is implemented in PyTorch \cite{paszke2019pytorch}, and will be publicly released upon acceptance.

\textbf{Baselines.}
We consider different classical and state-of-the-art GNN baselines. Specifically:
\begin{itemize}[leftmargin=3em]
    \item Classical MPNNs, \ie GCN~\citep{Kipf2016}, GraphSAGE~\citep{SAGE}, GAT~\citep{veličković2018graph}, GatedGCN~\citep{gatedgcn}, GIN~\citep{xu2018how}, GINE~\citep{gine}, GCNII~\citep{gcnii}, and CoGNN~\citep{finkelshtein2024cooperative};
    \item Heterophily-specific models, \ie H2GCN \citep{h2gcn}, CPGNN \citep{CPGNN}, FAGCN \citep{FAGCN}, GPR-GNN \citep{GPR_GNN},
FSGNN \citep{FSGNN}, GloGNN \cite{GloGNN}, GBK-GNN \citep{GBK_GNN}, and JacobiConv \citep{jacobiconv};
    \item DE-DGNs, \ie DGC~\citep{DGC}, GRAND~\citep{chamberlain2021grand}, GraphCON~\citep{rusch2022graph}, A-DGN~\citep{gravina_adgn}, and {SWAN}~\citep{gravina_swan};
    \item Graph Transformers, \ie Transformer~\citep{Vaswani2017, dwivedi2021generalization}, GT~\citep{ijcai2021p0214}, SAN~\citep{san}, GPS~\citep{graphgps}, GOAT~\citep{goat},  and Exphormer~\citep{Shirzad2023};
    \item Higher-Order DGNs, \ie DIGL~\citep{DIGL}, MixHop~\citep{abu2019mixhop}, and DRew~\citep{drew}.
        \item SSM-based GNN, \ie Graph-Mamba~\citep{wang2024mamba}, GMN~\citep{behrouz2024graphmamba}, and GPS+Mamba~\citep{behrouz2024graphmamba}
\end{itemize} 

\subsection{Synthetic Example from \Cref{fig:signal_comparison}}
In the synthetic example in \Cref{fig:signal_comparison}, we demonstrate the effectiveness of \ourmethod{} in overcoming the oversquashing issue in GNNs. To do that, we consider a Barbell graph, where all node features are set to 0, besides the left-most node in the graph, which is set to 1, as shown in \Cref{fig:signal_comparison}(a). The goal is to allow the information to propagate through all nodes effectively. We do this by considering a gradient flow process of the Dirichlet energy using 50 layers (steps), as shown in \cref{fig:signal_comparison}(b), where it is noticeable that the information is now flowing to the right part in the graph, because of the bottleneck between the two cliques. However, as we show in \Cref{fig:signal_comparison}(c), by considering our \ourmethod{}, which utilizes both an energy flow as well as a tangential flow, it is possible to effectively propagate the information through all the nodes in the graphs.

\subsection{Graph Property Prediction}\label{app:details_gpp}
\textbf{Dataset.}
We construct our benchmark following the protocol introduced by \citet{gravina_adgn}. Graph instances are synthetically generated from a variety of canonical topologies, including Erd\H{o}s–R\'{e}nyi, Barabasi-Albert, caveman, tree, and grid models. Each graph consists of 25 to 35 nodes, with node features initialized as random identifiers sampled uniformly from the interval $[0, 1)$. The prediction targets encompass several structural tasks: computing the shortest paths from a source node, estimating node eccentricity, and determining graph diameter. The complete dataset contains 7,040 graphs, split into 5,120 for training, 640 for validation, and 1,280 for testing. 
These tasks inherently demand capturing long-range dependencies, as they involve global graph computations such as shortest path inference. As highlighted in \citet{gravina_adgn}, traditional algorithms like Bellman-Ford or Dijkstra's method require multiple rounds of message propagation, which motivates the need for expressive graph models. The benchmark graph families, such as caveman, tree, line, star, caterpillar, and lobster, frequently include structural bottlenecks that are known to induce oversquashing effects \citep{topping2022understanding}, posing additional challenges for message-passing-based GNNs.

\textbf{Experimental Setup.}
We adopt the same evaluation framework as \citet{gravina_adgn}, including datasets, training routines, and hyperparameter spaces. Model training is conducted using the Adam optimizer for up to 1500 epochs, with early stopping triggered after 100 consecutive epochs of no improvement on the validation Mean Squared Error (MSE). Hyperparameters are selected via grid search, and performance is averaged over 4 independent runs with different random seeds for weight initialization. A summary of the hyperparameter grid used in our experiments is provided in Table~\ref{tab:hyperparams}.

\subsection{Graph Benchmarks from \citet{dwivedi2023benchmarking}}

\textbf{Dataset.}
To comprehensively assess the capabilities of \ourmethod{}, we evaluate its performance on a diverse set of graph learning benchmarks curated by \citet{dwivedi2023benchmarking}. The benchmark suite includes: 
\textit{ZINC-12k}, a molecular regression dataset containing chemical compounds, where the goal is to predict the constrained solubility of each molecule. Graphs represent molecular structures, with atoms as nodes and chemical bonds as edges. Node and edge features encode atom types and bond types, respectively. 
\textit{MNIST} and \textit{CIFAR-10} superpixels are graph-structured versions of standard image classification datasets, where images are converted into sparse graphs of superpixels. Each superpixel forms a node, and edges are based on spatial adjacency. The tasks involve classifying digits (MNIST) and natural objects (CIFAR-10) based on graph-structured representations. 
\textit{CLUSTER} and \textit{PATTERN} are synthetic datasets designed to assess the relational inductive biases of graph neural networks. Both datasets are generated from a set of stochastic block models (SBMs). In \textit{CLUSTER}, the task is to group nodes by community, while \textit{PATTERN} involves identifying specific structural patterns within each graph. 
These datasets span a variety of domains: chemical, image, and synthetic graphs, and are commonly used to benchmark architectural innovations in GNNs \citep{ma2023graph}. We follow the official training, validation, and test splits provided by \citet{dwivedi2023benchmarking}, ensuring consistency in evaluation across models.

\textbf{Experimental Setup.}
We adhere to the training and evaluation protocol established in \citet{dwivedi2023benchmarking}. For each dataset, we perform hyperparameter tuning via grid search, optimizing the corresponding evaluation metrics: Mean Absolute Error (MAE) for \textit{ZINC-12k}, and classification accuracy for the remaining tasks. We use the AdamW optimizer and train all models for up to 300 epochs, with early stopping based on validation performance. 
To ensure comparability with prior work, we respect the same parameter budgets used in the original benchmark and maintain the architectural constraints defined for fair evaluation. Each configuration is trained with three random seeds, and we report the average and standard deviation of the results. Hyperparameter ranges used in this set of experiments are summarized in Table~\ref{tab:hyperparams}.

\subsection{Long Range Graph Benchmark}\label{app:details_lrgb}
\textbf{Dataset.}
To evaluate model performance on real-world graphs with significant long-range dependencies, we utilize the \textit{Peptides-func} and \textit{Peptides-struct} benchmarks introduced in \citet{LRGB}. These datasets represent peptide molecules as graphs, where nodes correspond to heavy (non-hydrogen) atoms, and edges denote chemical bonds.  
\textit{Peptides-func} is a multi-label classification task with 10 functional categories, including antibacterial, antiviral, and signaling-related properties. In contrast, \textit{Peptides-struct} focuses on regression, targeting physical and geometric attributes such as molecular inertia (weighted by atomic mass and valence), atom pair distance extremes, sphericity, and average deviation from a best-fit plane. 
Together, the two datasets comprise 15,535 peptide graphs and roughly 2.3 million nodes. We adopt the official train/validation/test partitions from \citet{LRGB} and report mean and standard deviation across three different random seeds for each experiment.

\textbf{Experimental Setup.}
We follow the evaluation protocol established in \citet{LRGB}, including dataset usage, training strategy, and model capacity constraints. Hyperparameter tuning is carried out via grid search, optimizing for Average Precision (AP) in the classification task and Mean Absolute Error (MAE) in the regression task. All models are trained using the AdamW optimizer for up to 300 epochs, with early stopping based on validation performance. 
To ensure fairness and comparability, all models adhere to the 500K parameter limit, in line with the settings of \citet{LRGB} and \citet{drew}. Each configuration is run three times with different weight initializations, and results are averaged. Details of the hyperparameter ranges considered can be found in Table~\ref{tab:hyperparams}.

\subsection{Heterophilic Node Classification }\label{app:details_gnn_bench}
\textbf{Dataset.}
For evaluating performance in heterophilic graph settings, we consider five benchmark tasks introduced by \citet{platonov2023a}: \textit{Roman-Empire}, \textit{Amazon-Ratings}, \textit{Minesweeper}, \textit{Tolokers}, and \textit{Questions}. These datasets span a diverse range of domains and graph topologies. 
\textit{Roman-Empire} is constructed from the Wikipedia article on the Roman Empire, where nodes represent words and edges capture either sequential adjacency or syntactic relations. The task is node classification with 18 syntactic categories, and the underlying graph is sparse and chain-structured, suggesting the presence of long-range dependencies. 
\textit{Amazon-Ratings} originates from Amazon's product co-purchasing graph. Nodes correspond to products, linked if they are frequently bought together. The classification task involves predicting discretized average product ratings (five classes), with node features derived from fastText embeddings of product descriptions. 
\textit{Minesweeper} is a synthetic dataset modeled as a $100 \times 100$ grid. Nodes represent individual cells, with edges connecting adjacent cells. A random 20\% of nodes are labeled as mines, and the objective is to classify mine-containing cells based on one-hot features that encode the number of neighboring mines. 
\textit{Tolokers} is based on the Toloka crowdsourcing platform \citep{Tolokers}, where each node is a worker (toloker), and edges indicate co-participation on the same project. The task involves binary classification to detect whether a worker has been banned, using node features from user profiles and performance metrics. 
\textit{Questions} draws from user interaction data on Yandex Q, a question-answering forum. Nodes represent users, and edges capture answering interactions. The goal is to identify users who remain active, with input features derived from user-provided descriptions. 
A summary of dataset statistics is provided in Table~\ref{tab:hetero_stats}.
\begin{table}[h]

\centering
\caption{Statistics of the heterophilic node classification datasets.}
\label{tab:hetero_stats}
\scriptsize
\begin{tabular}{lccccc}
\hline\toprule
               & \textbf{Roman-empire} & \textbf{Amazon-ratings} & \textbf{Minesweeper} & \textbf{Tolokers} & \textbf{Questions}\\\midrule
N. nodes          & 22,662        & 24,492          & 10,000       & 11,758    & 48,921\\
N. edges          & 32,927        & 93,050          & 39,402       & 519,000   & 153,540\\
Avg degree     & 2.91         & 7.60           & 7.88        & 88.28    & 6.28\\
Diameter       & 6,824         & 46             & 99          & 11       & 16\\
Node features  & 300          & 300            & 7           & 10       & 301\\
Classes        & 18           & 5              & 2           & 2        & 2\\
Edge homophily & 0.05         & 0.38           & 0.68        & 0.59     & 0.84\\
\bottomrule\hline      
\end{tabular}
\end{table}

\textbf{Experimental Setup.}
Our experimental procedure aligns with that of \citet{malnet} and \citet{platonov2023a}. We conduct a grid search to optimize model performance, using classification accuracy for the \textit{Roman-Empire} and \textit{Amazon-Ratings} tasks, and ROC-AUC for \textit{Minesweeper}, \textit{Tolokers}, and \textit{Questions}. Each model is trained using the AdamW optimizer for a maximum of 300 epochs. 
Our experiments follow the official dataset splits provided by \citet{platonov2023a}. For each model configuration, we perform multiple training runs with different random seeds and report the mean and standard deviation of the results. The hyperparameter grid explored in these experiments is summarized in Table~\ref{tab:hyperparams}.

\subsection{Hyperparameters}\label{app:hyperparams}
In Table~\ref{tab:hyperparams}, we summarize the hyperparameter grids used for tuning our \ourmethod{} across different benchmarks. Alongside standard training hyperparameters such as learning rate, weight decay, and batch size, our method introduces several additional components. These include the number of unrolled steps $L$ (corresponding to the depth of the energy-based dynamics), the hidden dimension $d$ of node features, and the number of message-passing layers $L_{\text{gnn}}$ used within the internal $\textsc{EnergyGNN}$ and $\textsc{TangentGNN}$ modules.  
In all experiments, we share the architecture depth between $\textsc{EnergyGNN}$ and $\textsc{TangentGNN}$. We also tune the step size $\epsilon$ used in the forward Euler update (\Cref{eq:euler_update}), which controls the integration scale of the continuous dynamics. 
We explore multiple values of $L$ to assess how the number of dynamical steps impacts long-range propagation across different tasks. Details of the complete hyperparameter grid can be found in Table~\ref{tab:hyperparams}.

\begin{table}[h]
\centering
\caption{Hyperparameter grids used during model selection for the different benchmark categories: \emph{GraphPropPred} (Diameter, SSSP, Eccentricity), \emph{LRGB} (Peptides-func/struct), \emph{Graph Benchmarks} (ZINC-12k, MNIST, CIFAR-10, CLUSTER, PATTERN), and \emph{Node Classification} (Roman-Empire, Amazon-Ratings, Minesweeper, Tolokers, Questions).}
\label{tab:hyperparams}
\vspace{5pt}
\scriptsize
\begin{tabular}{l|l|l|l|l}
\toprule
\textbf{Hyperparameter} & \textbf{\emph{GraphPropPred}} & \textbf{\emph{LRGB}} & \textbf{\emph{Graph Benchmarks}} & \textbf{\emph{Node Classification}} \\
\midrule
Unrolled steps $L$ & \{1,5,10,20\} & \{2,4,8,16,32\} & \{2,4,8,16,32\} & \{2,4,8,16,32\} \\
GNN layers $L_{\text{gnn}}$ & \{1,2,4,8,16\} & \{1,2,4,8,16\} & \{1,2,4,8,16\} & \{1,2,4,8,16\} \\
Feature dimension $d$ & \{10, 20, 30\} & \{64, 128,256\} & \{64, 128, 256\} & \{64, 128, 256\} \\
Step size $\epsilon$ & \{0.001, 0.1, 1.0\} & \{0.001, 0.1, 1.0\} & \{0.001, 0.1, 1.0\} & \{0.001, 0.1, 1.0\} \\
Learning rate & \{1e-3, 1e-4\} & \{1e-3, 1e-4\} & \{1e-3, 1e-4\} & \{1e-3, 1e-4\} \\
Weight decay & \{0,1e-6, 1e-5\} & \{0, 1e-6, 1e-5\} & \{0, 1e-6, 1e-5\} & \{0, 1e-6, 1e-5\} \\
Activation function ($\sigma$)    &  ReLU & ELU, GELU, ReLU & ELU, GELU, ReLU &  ELU, GELU, ReLU\\

Batch size & \{32,64,128\} & \{32,64,128\} & \{32, 64,128\} & N/A \\
\bottomrule
\end{tabular}
\end{table}

\section{Additional Results and Comparisons}
\label{app:additional_results}

\subsection{Heterophilic Node Classification}
\label{app:heterophilic_results}
We report and compare the performance of our \ourmethod{} with other recent benchmarks on the heterophilic node classification datasets from \citet{platonov2023a}, in \Cref{tab:results_hetero}.  As can be seen from the Table, \ourmethod{} offers strong performance that is similar or better than recent state-of-the-art methods, further demonstrating its effectiveness.

\begin{table}[t]
\footnotesize
\setlength{\tabcolsep}{2pt}
\centering
\caption{Mean test set score and std averaged over the splits from \citet{platonov2023a}.
\one{First}, \two{second}, and \three{third}  best results for each task are color-coded. We mark each method once -- if two variants are among the leading methods, we mark the best-performing variant. 
}
\label{tab:results_hetero}
\begin{tabular}{lccccc}
\hline\toprule
\multirow{2}{*}{\textbf{Model}} & \textbf{Roman-empire} & \textbf{Amazon-ratings} & \textbf{Minesweeper} & \textbf{Tolokers} & \textbf{Questions}\\
& \scriptsize{Acc $\uparrow$} & \scriptsize{Acc $\uparrow$} & \scriptsize{AUC $\uparrow$} & \scriptsize{AUC $\uparrow$} & \scriptsize{AUC $\uparrow$}\\
\midrule
\textbf{MPNNs} \\
$\,$ GAT         & 80.87$_{\pm0.30}$ & 49.09$_{\pm0.63}$ & 92.01$_{\pm0.68}$ & 83.70$_{\pm0.47}$ & 77.43$_{\pm1.20}$\\
$\,$ GAT-sep     & 88.75$_{\pm0.41}$ & 52.70$_{\pm0.62}$ & \three{93.91$_{\pm0.35}$} & 83.78$_{\pm0.43}$ & 76.79$_{\pm0.71}$\\
$\,$ Gated-GCN & 74.46$_{\pm0.54}$ & 43.00$_{\pm0.32}$ &  87.54$_{\pm1.22}$ &  77.31$_{\pm1.14}$ & 76.61$_{\pm{1.13}}$\\
$\,$ GCN         & 73.69$_{\pm0.74}$ & 48.70$_{\pm0.63}$ & 89.75$_{\pm0.52}$ & 83.64$_{\pm0.67}$ & 76.09$_{\pm1.27}$\\
$\,$ CO-GNN($\Sigma$, $\Sigma$) & \two{91.57$_{\pm0.32}$} &  51.28$_{\pm0.56}$ &  95.09$_{\pm1.18}$ &  83.36$_{\pm0.89}$ & \two{80.02$_{\pm0.86}$} \\
$\,$ CO-GNN($\mu$, $\mu$) & 91.37$_{\pm0.35}$ &  \one{54.17$_{\pm0.37}$} &  \two{97.31$_{\pm0.41}$} &  \two{84.45$_{\pm1.17}$} & 76.54$_{\pm0.95}$ \\
$\,$ SAGE        & 85.74$_{\pm0.67}$ & \three{53.63$_{\pm0.39}$} & 93.51$_{\pm0.57}$ & 82.43$_{\pm0.44}$ & 76.44$_{\pm0.62}$\\
\midrule
\textbf{Graph Transformers} \\
$\,$ {Exphormer}   & \three{89.03$_{\pm0.37}$}  &  53.51$_{\pm0.46}$  &  90.74$_{\pm0.53}$  &  83.77$_{\pm0.78}$ & 73.94$_{\pm1.06}$\\
$\,$ NAGphormer  & 74.34$_{\pm0.77}$  &  51.26$_{\pm0.72}$  &  84.19$_{\pm0.66}$  &  78.32$_{\pm0.95}$ & 68.17$_{\pm1.53}$\\
$\,$ GOAT        & 71.59$_{\pm1.25}$  &  44.61$_{\pm0.50}$  &  81.09$_{\pm1.02}$  &  83.11$_{\pm1.04}$ & 75.76$_{\pm1.66}$\\
$\,$ GPS\textsubscript{GAT+Performer} (RWSE) & 87.04$_{\pm0.58}$ & 49.92$_{\pm0.68}$ & 91.08$_{\pm0.58}$ & \three{84.38$_{\pm0.91}$} & 77.14$_{\pm1.49}$\\
$\,$ GT          & 86.51$_{\pm0.73}$ & 51.17$_{\pm0.66}$ & 91.85$_{\pm0.76}$ & 83.23$_{\pm0.64}$ & 77.95$_{\pm0.68}$\\
$\,$ GT-sep      & 87.32$_{\pm0.39}$ & 52.18$_{\pm0.80}$ & 92.29$_{\pm0.47}$ & 82.52$_{\pm0.92}$ & 78.05$_{\pm0.93}$\\
\midrule
\multicolumn{4}{l}{\textbf{Heterophily-Designated GNNs}} \\
$\,$ FAGCN       & 65.22$_{\pm0.56}$ & 44.12$_{\pm0.30}$ & 88.17$_{\pm0.73}$ & 77.75$_{\pm1.05}$ & 77.24$_{\pm1.26}$\\
$\,$ FSGNN       & 79.92$_{\pm0.56}$ & 52.74$_{\pm0.83}$ & 90.08$_{\pm0.70}$ & 82.76$_{\pm0.61}$ & \three{78.86$_{\pm0.92}$}\\
$\,$ GBK-GNN     & 74.57$_{\pm0.47}$ & 45.98$_{\pm0.71}$ & 90.85$_{\pm0.58}$ & 81.01$_{\pm0.67}$ & 74.47$_{\pm0.86}$\\
$\,$ GloGNN      & 59.63$_{\pm0.69}$ & 36.89$_{\pm0.14}$ & 51.08$_{\pm1.23}$ & 73.39$_{\pm1.17}$ & 65.74$_{\pm1.19}$\\
$\,$ GPR-GNN     & 64.85$_{\pm0.27}$ & 44.88$_{\pm0.34}$ & 86.24$_{\pm0.61}$ & 72.94$_{\pm0.97}$ & 55.48$_{\pm0.91}$\\
$\,$ JacobiConv  & 71.14$_{\pm0.42}$ & 43.55$_{\pm0.48}$ & 89.66$_{\pm0.40}$ & 68.66$_{\pm0.65}$ & 73.88$_{\pm1.16}$\\
\midrule
\textbf{Ours} \\

$\,$ \ourmethod\textsubscript{GatedGCN} & \one{91.89$_{\pm0.30}$} & {52.60$_{\pm0.53}$} & {98.32$_{\pm0.59}$} & {85.51$_{\pm0.98}$} & {80.39$_{\pm1.04}$} \\

$\,$  \ourmethod\textsubscript{GPS} &  {91.08$_{\pm0.57}$} & \two{53.83$_{\pm0.32}$} & \one{98.39$_{\pm0.54}$} & \one{85.66$_{\pm1.01}$} & \one{80.32$_{\pm1.07}$} \\

\bottomrule\hline      
\end{tabular}
\end{table}

\subsection{Additional Comparisons}
\label{app:additional_comparisons}
The comparisons made in \Cref{sec:experiments} offer a focused comparison with directly related methods as well as baseline backbones. In addition to that, we now provide a more comprehensive comparison in \Cref{tab:results_lrgb_complete} and \Cref{tab:results_hetero_complete}, to further facilitate a comprehensive comparison with recent methods. As can be seen, also under these comparisons, our \ourmethod{} offers strong performance. 

\subsection{Ablation Study}
\label{sec:ablation}

\textbf{Setup.} We conduct two key ablation studies to better understand the contributions of the energy function and the tangential flow in \ourmethod{}. Specifically, we aim to answer the following questions: 

(i) \emph{Does downstream performance benefit from incorporating a tangential term even when the underlying GNN is not the gradient of an energy function?}

(ii) \emph{Is the observed improvement due to the tangential nature of the added component, or simply due to additional parameters and network?}

To address these questions, we design two controlled experiments. For comprehensive coverage, we evaluate one representative dataset from each benchmark group: ZINC-12k, Roman-empire, Peptides-func, and Diameter. All experiments are run with two backbone architectures, GatedGCN and GPS. For reference, we also report the performance of the original backbones.

\textbf{Results.} For ablation (i), we compare \ourmethod{} against a variant we call \ourmethod{}-\textsc{NON-ENERGY}, in which the gradient-based energy descent term $\nabla_{\bfH} V_{\mathcal{G}}(\bfH^{(\ell)})$ in \Cref{eq:euler_update} is replaced by intermediate node features from the same GNN backbone, as detailed in \Cref{eq:intermediate_energy_features}. These features are computed using the same architecture but are not guaranteed to correspond to the gradient of any scalar energy function. This setup ensures fairness in capacity while removing the energy-based structure. As shown in \Cref{tab:ablation_energy}, although both variants benefit from the inclusion of the tangential component, the full \ourmethod{} consistently outperforms \ourmethod{}-\textsc{NON-ENERGY}, confirming that leveraging a valid energy gradient contributes meaningfully to downstream performance.

For ablation (ii), we isolate the effect of the tangential nature of the added direction. In this variant, denoted \ourmethod{}-\textsc{NON-TANGENT}, we use the same output from the tangential network as in \Cref{eq:output_tan_MPNN} but omit the orthogonal projection step defined in \Cref{eq:tangential_flow}. Thus, while we still introduce an additional GNN term into the dynamics, it is not explicitly orthogonal to the energy gradient. Our results in \Cref{tab:ortho_importance} show that while this variant improves the performance compared with the baseline backbone, it also results in a drop in performance compared to the full \ourmethod{}. This highlights the importance of the tangential constraint, and its contribution towards improving the utilization of the learned energy function, as discussed in \Cref{sec:theoretical-analysis}.
Together, these ablations underscore the importance of both components in our design: (i) the principled learned energy descent, and  (ii) the structured tangential update, as crucial for effective and flexible feature evolution.

\begin{table}[t]
\footnotesize
    \centering
        \caption{Ablation study on the importance of using a gradient of an energy term in \Cref{eq:euler_update}.}

    \begin{tabular}{lcccc}
    \hline
    \toprule 
    \multirow{2}{*}{\textbf{Model}}  & \textbf{ZINC-12k} & \textbf{Roman-empire} & \textbf{Peptides-func} & \textbf{Diameter} \\
    & \scriptsize{MAE $\downarrow$} & \scriptsize{Acc. $\uparrow$}
    & 
    \scriptsize{AP $\uparrow$}
    & 
    \scriptsize{$\log_{10}$(MSE) $\downarrow$}
    \\ 
    \midrule
        GatedGCN & 0.282$_{\pm 0.015}$ &   74.46$_{\pm0.54}$ & 58.64$_{\pm0.77}$ & 0.1348$_{\pm 0.0397}$ \\

       \ourmethod{}-\textsc{NON-ENERGY}\textsubscript{GatedGCN}  & 0.138$_{\pm 0.014}$ & 86.94$_{\pm0.43}$ &  68.07$_{\pm0.45}$  & -0.5992$_{\pm0.0831}$\\
       \ourmethod{}\textsubscript{GatedGCN}  & \textbf{0.128$_{\pm 0.011}$}  & \textbf{91.89$_{\pm0.30}$} & \textbf{68.92$_{\pm0.40}$} & \textbf{-0.6681$_{\pm0.0745}$}\\
       \midrule
              GPS & 0.070$_{\pm 0.004}$ &   87.04$_{\pm0.58}$ & 65.35$_{\pm0.41}$ & -0.5121$_{\pm0.0426}$ \\ 

              \ourmethod{}-\textsc{NON-ENERGY}\textsubscript{GPS}  & 0.067$_{\pm 0.004}$ &  89.00$_{\pm0.61}$ &  67.58$_{\pm0.39}$ & -0.7178$_{\pm0.0729}$ \\
       \ourmethod{}\textsubscript{GPS}  & \textbf{0.062$_{\pm 0.005}$}  & \textbf{91.08$_{\pm0.57}$} & \textbf{70.21$_{\pm0.43}$} & \textbf{-0.9772$_{\pm0.0518}$}\\
        \bottomrule
        \hline
    \end{tabular}
    \label{tab:ablation_energy}
\end{table}

\begin{table}[t]
\footnotesize
    \centering
        \caption{The importance of using a tangential term to the  energy term in \Cref{eq:euler_update}.}

    \begin{tabular}{lcccc}
    \hline
    \toprule 
    \multirow{2}{*}{\textbf{Model}}  & \textbf{ZINC-12k} & \textbf{Roman-empire} & \textbf{Peptides-func} & \textbf{Diameter} \\
    & \scriptsize{MAE $\downarrow$} & \scriptsize{Acc. $\uparrow$}
    & 
    \scriptsize{AP $\uparrow$}
    & 
    \scriptsize{$\log_{10}$(MSE) $\downarrow$}
    \\ 
    \midrule
    GatedGCN & 0.282$_{\pm 0.015}$ &   74.46$_{\pm0.54}$ & 58.64$_{\pm0.77}$ & 0.1348$_{\pm 0.0397}$ \\
       \ourmethod{}-\textsc{NON-TANGENT}\textsubscript{GatedGCN}  & 
       0.186$_{\pm 0.016}$  & {83.59$_{\pm0.48}$} & {68.01$_{\pm0.52}$} & {-0.2193$_{\pm0.0899}$}
      \\ \ourmethod{}\textsubscript{GatedGCN}  & \textbf{0.128$_{\pm 0.011}$}  & \textbf{91.89$_{\pm0.30}$} & \textbf{68.92$_{\pm0.40}$} & \textbf{-0.6681$_{\pm0.0745}$}\\
       \midrule
       GPS & 0.070$_{\pm 0.004}$ &   87.04$_{\pm0.58}$ & 65.35$_{\pm0.41}$ & -0.5121$_{\pm0.0426}$ \\ 
              \ourmethod{}-\textsc{NON-TANGENT}\textsubscript{GPS}  & {0.066$_{\pm 0.010}$}  & {88.57$_{\pm0.72}$} & {67.33$_{\pm0.59}$} & {-0.2916$_{\pm0.0404}$}  \\
       \ourmethod{}\textsubscript{GPS}  & \textbf{0.062$_{\pm 0.005}$}  & \textbf{91.08$_{\pm0.57}$} & \textbf{70.21$_{\pm0.43}$} & \textbf{-0.9772$_{\pm0.0518}$}\\
        \bottomrule
        \hline
    \end{tabular}
    \label{tab:ortho_importance}
\end{table}

\begin{table}[t]
\footnotesize
    \centering
    \caption{Results for Peptides-func and Peptides-struct averaged over 3 training seeds. Baseline results are taken from \cite{LRGB} and \cite{drew}. Re-evaluated methods employ the 3-layer MLP readout proposed in \cite{tonshoff2023where}. Note that all MPNN-based methods include structural and positional encoding.
   $^\ddag$ means 3-layer MLP readout and residual connections are employed based on \citep{tonshoff2023where}. This table is an extended version of the focused \Cref{tab:results_lrgb}.
    }\label{tab:results_lrgb_complete}
    
    \begin{tabular}{@{}lcc@{}}
    \hline\toprule
    \multirow{2}{*}{\textbf{Model}} & \textbf{Peptides-func}  & \textbf{Peptides-struct}              
    \\
    & \scriptsize{AP $\uparrow$} & \scriptsize{MAE $\downarrow$}                            
    \\ \midrule  
    \textbf{MPNNs} \\
    $\,$ GCN                        & 59.30$_{\pm0.23}$ & 0.3496$_{\pm0.0013}$ \\ 
    $\,$ GINE                       & 54.98$_{\pm0.79}$ & 0.3547$_{\pm0.0045}$ \\ 
    $\,$ GCNII                      & 55.43$_{\pm0.78}$ & 0.3471$_{\pm0.0010}$ \\
    $\,$ GatedGCN                   & 58.64$_{\pm0.77}$ & 0.3420$_{\pm0.0013}$ \\ 
    
    \midrule
    \textbf{Multi-hop GNNs}\\
    $\,$ DIGL+MPNN           & 64.69$_{\pm0.19}$         & 0.3173$_{\pm0.0007}$ \\ 
    $\,$ DIGL+MPNN+LapPE     & 68.30$_{\pm0.26}$         & 0.2616$_{\pm0.0018}$ \\ 
    $\,$ MixHop-GCN          & 65.92$_{\pm0.36}$         & 0.2921$_{\pm0.0023}$ \\ 
    $\,$ MixHop-GCN+LapPE    & 68.43$_{\pm0.49}$         & 0.2614$_{\pm0.0023}$ \\
    $\,$ DRew-GCN            & 69.96$_{\pm0.76}$         & 0.2781$_{\pm0.0028}$ \\
    $\,$ DRew-GCN+LapPE             & {71.50$_{\pm0.44}$}   & 0.2536$_{\pm0.0015}$ \\ 
    $\,$ DRew-GIN            & 69.40$_{\pm0.74}$         & 0.2799$_{\pm0.0016}$ \\ 
    $\,$ DRew-GIN+LapPE      & {71.26$_{\pm0.45}$}   & 0.2606$_{\pm0.0014}$ \\ 
    $\,$ DRew-GatedGCN       & 67.33$_{\pm0.94}$         & 0.2699$_{\pm0.0018}$ \\
    $\,$ DRew-GatedGCN+LapPE & 69.77$_{\pm0.26}$         & 0.2539$_{\pm0.0007}$ \\ 
    \midrule
    
    \textbf{Transformers} \\
    $\,$ Transformer+LapPE & 63.26$_{\pm1.26}$ & 0.2529$_{\pm0.0016}$           \\ 
    $\,$ SAN+LapPE         & 63.84$_{\pm1.21}$ & 0.2683$_{\pm0.0043}$           \\ 
    $\,$ GraphGPS+LapPE    & 65.35$_{\pm0.41}$ & 0.2500$_{\pm0.0005}$   \\  
    \midrule
    \textbf{Modified and Re-evaluated}$^\ddag$\\
    $\,$ GCN 
    & 68.60$_{\pm0.50}$ & 0.2460$_{\pm0.0007}$\\
    $\,$ GINE     
    & 66.21$_{\pm0.67}$ & 0.2473$_{\pm0.0017}$\\
    $\,$ GatedGCN 
    & 67.65$_{\pm0.47}$ & 0.2477$_{\pm0.0009}$\\
    $\,$ GraphGPS & 65.34$_{\pm0.91}$ & 0.2509$_{\pm0.0014}$\\
   
    \midrule
    \textbf{DE-GNNs} \\
    $\,$ GRAND      & 57.89$_{\pm0.62}$ & 0.3418$_{\pm0.0015}$ \\
    $\,$ GraphCON   & 60.22$_{\pm0.68}$ & 0.2778$_{\pm0.0018}$ \\ 
    $\,$ A-DGN      & 59.75$_{\pm0.44}$ & 0.2874$_{\pm0.0021}$ \\
    $\,$ {SWAN} & 67.51$_{\pm0.39}$ & 0.2485$_{\pm0.0009}$\\
    \midrule
    \textbf{Graph SSMs} \\
    $\,$ Graph-Mamba & 67.39$_{\pm0.87}$ &  {0.2478$_{\pm0.0016}$} \\
    $\,$ GMN & 70.71$_{\pm0.83}$ &  {0.2473$_{\pm0.0025}$} \\
    \midrule
    \textbf{Ours} \\
    
    $\,$ \ourmethod$_{\textsc{GatedGCN}}$ & 68.92$_{\pm0.40}$ & 0.2451$_{\pm0.0006}$  \\
        $\,$ \ourmethod$_{\textsc{GPS}}$ & 70.21$_{\pm0.43}$  & 0.2422$_{\pm0.0014}$  \\

    \bottomrule\hline
    \end{tabular}
    \end{table}

    \begin{table}[t]
\footnotesize
\setlength{\tabcolsep}{2pt}

\centering
\caption{Mean test set score and std averaged over the splits from \citet{platonov2023a}. This table is an extended version of the focused \Cref{tab:results_hetero}. Baseline results are reported from \cite{finkelshtein2024cooperative, platonov2023a, muller2024attending, luan2024heterophilicgraphlearninghandbook}. 
}
\label{tab:results_hetero_complete}
\begin{tabular}{lccccc}
\hline\toprule
\multirow{2}{*}{\textbf{Model}} & \textbf{Roman-empire} & \textbf{Amazon-ratings} & \textbf{Minesweeper} & \textbf{Tolokers} & \textbf{Questions}\\
& \scriptsize{Acc $\uparrow$} & \scriptsize{Acc $\uparrow$} & \scriptsize{AUC $\uparrow$} & \scriptsize{AUC $\uparrow$} & \scriptsize{AUC $\uparrow$}\\
\midrule
\textbf{MPNNs} \\
$\,$ GAT         & 80.87$_{\pm0.30}$ & 49.09$_{\pm0.63}$ & 92.01$_{\pm0.68}$ & 83.70$_{\pm0.47}$ & 77.43$_{\pm1.20}$\\
$\,$ GAT-sep     & 88.75$_{\pm0.41}$ & 52.70$_{\pm0.62}$ & {93.91$_{\pm0.35}$} & 83.78$_{\pm0.43}$ & 76.79$_{\pm0.71}$\\
$\,$ GAT (LapPE) & 84.80$_{\pm0.46}$ & 44.90$_{\pm0.73}$ & 93.50$_{\pm0.54}$ & 84.99$_{\pm0.54}$ & 76.55$_{\pm0.84}$\\
$\,$ GAT (RWSE) & 86.62$_{\pm0.53}$ & 48.58$_{\pm0.41}$ & 92.53$_{\pm0.65}$ & 85.02$_{\pm0.67}$ & 77.83$_{\pm1.22}$\\
$\,$ GAT (DEG) & 85.51$_{\pm0.56}$ & 51.65$_{\pm0.60}$ & 93.04$_{\pm0.62}$ & 84.22$_{\pm0.81}$ & 77.10$_{\pm1.23}$\\
$\,$ Gated-GCN & 74.46$_{\pm0.54}$ & 43.00$_{\pm0.32}$ &  87.54$_{\pm1.22}$ &  77.31$_{\pm1.14}$ & 76.61$_{\pm{1.13}}$\\
$\,$ GCN         & 73.69$_{\pm0.74}$ & 48.70$_{\pm0.63}$ & 89.75$_{\pm0.52}$ & 83.64$_{\pm0.67}$ & 76.09$_{\pm1.27}$\\
$\,$ GCN (LapPE) & 83.37$_{\pm0.55}$ & 44.35$_{\pm0.36}$ & 94.26$_{\pm0.49}$ & 84.95$_{\pm0.78}$ & 77.79$_{\pm1.34}$\\
$\,$ GCN (RWSE) & 84.84$_{\pm0.55}$ & 46.40$_{\pm0.55}$ & 93.84$_{\pm0.48}$ & 85.11$_{\pm0.77}$ & 77.81$_{\pm1.40}$\\
$\,$ GCN (DEG) & 84.21$_{\pm0.47}$ & 50.01$_{\pm0.69}$ & 94.14$_{\pm0.50}$ & 82.51$_{\pm0.83}$ & 76.96$_{\pm1.21}$\\
$\,$ CO-GNN($\Sigma$, $\Sigma$) & {91.57$_{\pm0.32}$} &  51.28$_{\pm0.56}$ &  95.09$_{\pm1.18}$ &  83.36$_{\pm0.89}$ & {80.02$_{\pm0.86}$} \\
$\,$ CO-GNN($\mu$, $\mu$) & 91.37$_{\pm0.35}$ &  {54.17$_{\pm0.37}$} &  {97.31$_{\pm0.41}$} &  {84.45$_{\pm1.17}$} & 76.54$_{\pm0.95}$ \\
$\,$ SAGE        & 85.74$_{\pm0.67}$ & {53.63$_{\pm0.39}$} & 93.51$_{\pm0.57}$ & 82.43$_{\pm0.44}$ & 76.44$_{\pm0.62}$\\
\midrule
\textbf{Graph Transformers} \\
$\,$ {Exphormer}   & {89.03$_{\pm0.37}$}  &  53.51$_{\pm0.46}$  &  90.74$_{\pm0.53}$  &  83.77$_{\pm0.78}$ & 73.94$_{\pm1.06}$\\
$\,$ NAGphormer  & 74.34$_{\pm0.77}$  &  51.26$_{\pm0.72}$  &  84.19$_{\pm0.66}$  &  78.32$_{\pm0.95}$ & 68.17$_{\pm1.53}$\\
$\,$ GOAT        & 71.59$_{\pm1.25}$  &  44.61$_{\pm0.50}$  &  81.09$_{\pm1.02}$  &  83.11$_{\pm1.04}$ & 75.76$_{\pm1.66}$\\
$\,$ GPS         & 82.00$_{\pm0.61}$  &  53.10$_{\pm0.42}$  &  90.63$_{\pm0.67}$  &  83.71$_{\pm0.48}$ & 71.73$_{\pm1.47}$\\
$\,$ GPS\textsubscript{GCN+Performer} (LapPE) & 83.96$_{\pm0.53}$ & 48.20$_{\pm0.67}$ & 93.85$_{\pm0.41}$ & 84.72$_{\pm0.77}$ & 77.85$_{\pm1.25}$\\
$\,$ GPS\textsubscript{GCN+Performer} (RWSE) & 84.72$_{\pm0.65}$ & 48.08$_{\pm0.85}$ & 92.88$_{\pm0.50}$ & 84.81$_{\pm0.86}$ & 76.45$_{\pm1.51}$\\
$\,$ GPS\textsubscript{GCN+Performer} (DEG) & 83.38$_{\pm0.68}$ & 48.93$_{\pm0.47}$ & 93.60$_{\pm0.47}$ & 80.49$_{\pm0.97}$ & 74.24$_{\pm1.18}$\\
$\,$ GPS\textsubscript{GAT+Performer} (LapPE) & 85.93$_{\pm0.52}$ & 48.86$_{\pm0.38}$ & 92.62$_{\pm0.79}$ & 84.62$_{\pm0.54}$ & 76.71$_{\pm0.98}$\\
$\,$ GPS\textsubscript{GAT+Performer} (RWSE) & 87.04$_{\pm0.58}$ & 49.92$_{\pm0.68}$ & 91.08$_{\pm0.58}$ & {84.38$_{\pm0.91}$} & 77.14$_{\pm1.49}$\\
$\,$ GPS\textsubscript{GAT+Performer} (DEG) & 85.54$_{\pm0.58}$ & 51.03$_{\pm0.60}$ & 91.52$_{\pm0.46}$ & 82.45$_{\pm0.89}$ & 76.51$_{\pm1.19}$\\
$\,$ GPS\textsubscript{GCN+Transformer} (LapPE) & OOM &  OOM & 91.82$_{\pm0.41}$ & 83.51$_{\pm0.93}$ & OOM \\
$\,$ GPS\textsubscript{GCN+Transformer} (RWSE) & OOM &  OOM & 91.17$_{\pm0.51}$ & 83.53$_{\pm1.06}$ & OOM \\
$\,$ GPS\textsubscript{GCN+Transformer} (DEG) & OOM &  OOM & 91.76$_{\pm0.61}$ & 80.82$_{\pm0.95}$ & OOM \\
$\,$ GPS\textsubscript{GAT+Transformer} (LapPE) & OOM &  OOM & 92.29$_{\pm0.61}$ & 84.70$_{\pm0.56}$ & OOM \\
$\,$ GPS\textsubscript{GAT+Transformer} (RWSE) & OOM &  OOM & 90.82$_{\pm0.56}$ & 84.01$_{\pm0.96}$ & OOM \\
$\,$ GPS\textsubscript{GAT+Transformer} (DEG) & OOM &  OOM & 91.58$_{\pm0.56}$ & 81.89$_{\pm0.85}$ & OOM \\
$\,$ GT          & 86.51$_{\pm0.73}$ & 51.17$_{\pm0.66}$ & 91.85$_{\pm0.76}$ & 83.23$_{\pm0.64}$ & 77.95$_{\pm0.68}$\\
$\,$ GT-sep      & 87.32$_{\pm0.39}$ & 52.18$_{\pm0.80}$ & 92.29$_{\pm0.47}$ & 82.52$_{\pm0.92}$ & 78.05$_{\pm0.93}$\\
\midrule
\multicolumn{4}{l}{\textbf{Heterophily-Designated GNNs}} \\
$\,$ CPGNN       & 63.96$_{\pm0.62}$ & 39.79$_{\pm0.77}$ & 52.03$_{\pm5.46}$ & 73.36$_{\pm1.01}$ & 65.96$_{\pm1.95}$\\
$\,$ FAGCN       & 65.22$_{\pm0.56}$ & 44.12$_{\pm0.30}$ & 88.17$_{\pm0.73}$ & 77.75$_{\pm1.05}$ & 77.24$_{\pm1.26}$\\
$\,$ FSGNN       & 79.92$_{\pm0.56}$ & 52.74$_{\pm0.83}$ & 90.08$_{\pm0.70}$ & 82.76$_{\pm0.61}$ & {78.86$_{\pm0.92}$}\\
$\,$ GBK-GNN     & 74.57$_{\pm0.47}$ & 45.98$_{\pm0.71}$ & 90.85$_{\pm0.58}$ & 81.01$_{\pm0.67}$ & 74.47$_{\pm0.86}$\\
$\,$ GloGNN      & 59.63$_{\pm0.69}$ & 36.89$_{\pm0.14}$ & 51.08$_{\pm1.23}$ & 73.39$_{\pm1.17}$ & 65.74$_{\pm1.19}$\\
$\,$ GPR-GNN     & 64.85$_{\pm0.27}$ & 44.88$_{\pm0.34}$ & 86.24$_{\pm0.61}$ & 72.94$_{\pm0.97}$ & 55.48$_{\pm0.91}$\\
$\,$ H2GCN       & 60.11$_{\pm0.52}$ & 36.47$_{\pm0.23}$ & 89.71$_{\pm0.31}$ & 73.35$_{\pm1.01}$ & 63.59$_{\pm1.46}$\\
$\,$ JacobiConv  & 71.14$_{\pm0.42}$ & 43.55$_{\pm0.48}$ & 89.66$_{\pm0.40}$ & 68.66$_{\pm0.65}$ & 73.88$_{\pm1.16}$\\
\midrule
\textbf{Graph SSMs} \\
$\,$ GMN          & 87.69$_{\pm0.50}$  &  {54.07$_{\pm0.31}$}  &  91.01$_{\pm0.23}$  &  84.52$_{\pm0.21}$ & --\\
$\,$ GPS + Mamba  & 83.10$_{\pm0.28}$  &  45.13$_{\pm0.97}$  &  89.93$_{\pm0.54}$  &  83.70$_{\pm1.05}$ & --\\
\midrule
\textbf{Ours} \\

$\,$ \ourmethod\textsubscript{GatedGCN} & {91.89$_{\pm0.30}$} & {52.60$_{\pm0.53}$} & {98.32$_{\pm0.59}$} & {85.51$_{\pm0.98}$} & {80.39$_{\pm1.04}$} \\

$\,$  \ourmethod\textsubscript{GPS} &  {91.08$_{\pm0.57}$} & {53.83$_{\pm0.32}$} & {98.39$_{\pm0.54}$} & {85.66$_{\pm1.01}$} & {80.32$_{\pm1.07}$} \\

\bottomrule\hline      
\end{tabular}
\end{table}


\end{document}